\documentclass[11pt]{article}
\usepackage{float}
\usepackage{amsmath}
\usepackage{amsthm}
\usepackage{amssymb}
 \usepackage{mathrsfs}
\usepackage{pstricks}
\usepackage{url}
\usepackage{subfig}
\usepackage{enumerate}
\usepackage{graphicx}
\usepackage{comment}
\usepackage{url}
\usepackage[colorlinks,linkcolor=black,citecolor=blue,urlcolor=blue]{hyperref}
\usepackage[pagewise]{lineno}

\usepackage{tikz}
\newtheorem{defn}{Definition}
\newtheorem{thm}{Theorem}
\newtheorem{lem}{Lemma}
\newtheorem{exmp}{Example}

\newtheorem{prop}{Proposition}

\begin{document}

\title{On SCC-recursiveness in Quantitative Argumentation}

\author{Zongshun Wang, 
Yuping Shen$^*$\\
Institute of Logic and Cognition,\\ 
Department of Philosophy, \\
Sun Yat-sen University, Guangzhou, China\\
E-mail: wangzsh7@mail2.sysu.edu.cn; shyping@mail.sysu.edu.cn
}
  \date{}
\maketitle

\begin{abstract}

\emph{Abstract argumentation} is a reasoning model for evaluating arguments based on various \emph{semantics}. 
\emph{SCC-recursiveness} is a sophisticated property of semantics that provides a general schema for characterizing semantics through the decomposition along \emph{strongly connected components} (SCCs). 
While this property has been extensively explored in various qualitative frameworks, it has been relatively neglected in quantitative argumentation. 
To fill this gap, we demonstrate that this property is well-suited to \emph{fuzzy extension semantics}, which is a quantitative generalization of classical semantics in \emph{fuzzy argumentation frameworks} (FAF).
We tailor the SCC-recursive schema to enable the characterization of fuzzy extension semantics through the recursive decomposition of an FAF along its SCCs.
Our contributions are twofold.
Theoretically, we show that SCC-recursiveness provides an alternative approach to characterize fuzzy extension semantics, offering a deep understanding and better insight into these semantics.
Practically, our schema provides a sound and complete algorithm for computing fuzzy extension semantics, which naturally reduces computational efforts when dealing with a large number of SCCs.

\end{abstract}

\section{Introduction}
Argumentation serves as a process for reasoning and decision-making in conflict situations, garnering significant attention in the field of Artificial Intelligence \cite{bench2007argumentation,rahwan2009argumentation}.
Its applications span diverse domains, including reasoning with inconsistent information \cite{simari1992mathematical}, decision making \cite{Amgoud2009using}, explainable AI \cite{Cyras2021XAI},  etc.

Dung's seminal work on \textit{argumentation framework} \cite{dung1995acceptability} (AF) is a well-studied formalism in argumentation theory.
It abstracts argumentation scenarios as a directed graph where nodes represent arguments and arrows represent attacks among arguments. 
In recent years, Dung's AF has been extensively explored in quantitative settings, leading to significant developments such as Fuzzy AF \cite{da2011changing,janssen2008fuzzy,wu2016godel}, Probabilistic AF \cite{hunter2013probabilistic,li2011probabilistic}
 and Weighted AF \cite{Bistarelli2021weighted,dunne2011weighted}.
These quantitative frameworks enrich the expressive power of the classical AF by associating numerical values with arguments or attacks to capture uncertain information.

The evaluation of arguments is a central topic in argumentation literature, commonly achieved through various \emph{semantics} \cite{baroni2011introduction}.
For instance, the well-known \emph{extension semantics} \cite{dung1995acceptability} are proposed for Dung's AF to identify sets of jointly accepted arguments, while \emph{gradual semantics} \cite{da2011changing,Amgoud2017Acceptability,wang2024bilateral} are developed for quantitative frameworks to calculate the \emph{acceptability degree} of arguments.
Moreover, investigating the properties of semantics is crucial for the understanding, comparison and computation of semantics.
The properties of extension semantics and gradual semantics have been deeply examined in studies such as \cite{baroni2007principle,van2017principle} and \cite{baroni2019fine,amgoud2021Evaluation}, respectively.

\emph{SCC-recursiveness} \cite{baroni2005scc} is a sophisticated property of semantics that relies on the graph-theoretical notion of \emph{strongly connected components} (SCCs).
Its significance lies in providing a general schema for characterizing semantics through recursively decomposing an AF into its SCCs.
This schema has attracted extensive interest in the literature as it has been proven to be useful for characterizing new semantics \cite{baroni2005scc,dvovrak2016stage,villata2011attack}, reducing computational efforts \cite{cerutti2014scc,dvovrak2022recursion,Baroni2014On}, and conducting principle analysis \cite{van2017principle,yu2021principle,dvovrak2024principles}.
Recently, this property has also been extended to many other frameworks, such as \emph{ADF} \cite{gaggl2021decomposition}, \emph{SETAF} \cite{dvovrak2024principles}, \emph{Unrestricted AF} \cite{baumann2017study}, among others.

Despite substantial contributions on this topic, almost all existing research on SCC-recursiveness has been restricted to qualitative settings. 
In contrast, research in various quantitative settings remains \emph{open} for investigation \cite{rienstra2018probabilistic}.
This raises the question of how to define an SCC-recursive scheme to characterize semantics in quantitative frameworks, such as probabilistic, fuzzy or weighted.
One underlying challenge is that, in quantitative frameworks, arguments are often evaluated based on the degree to which they can be accepted, which seems inapplicable to the schema.
In this paper, we demonstrate that this property is well-suited to \emph{fuzzy extension semantics}, which is introduced in \cite{wu2016godel} to generalize classical extension semantics in \textit{fuzzy argumentation frameworks} (FAF).
We tailor the existing SCC-recursive schema to enable the characterization of fuzzy extension semantics---including \emph{admissible}\footnote{Following \cite{baroni2018semantics}, we adopt the term `admissible semantics' for the convenience of presentation, even though it is not considered as a semantics in Dung’s original work.}, \emph{complete}, \emph{grounded} and \emph{preferred}---through the recursive decomposition of an FAF along its SCCs.

Investigating the SCC-recursiveness of these semantics benefits their understanding and computation.
In line with the results in qualitative settings, our contributions are twofold.
Theoretically, we show that SCC-recursiveness provides an alternative approach to characterize fuzzy extension semantics, offering a deep understanding and better insight into these semantics.
Practically, our schema provides a sound and complete algorithm for computing fuzzy extension semantics, underpinned by several key theorems. 
As illustrated by examples, this algorithm naturally reduces computational efforts when dealing with graphs containing a large number of SCCs, laying an implementation foundation for real-world applications.


The remaining part of this paper is structured as follows.
Section \ref{Preliminary} reviews some basic concepts.
Section \ref{SCC-recursive in FAF} establishes the basic theory of SCC-recursiveness in FAF.
Section \ref{SCC-recursive semantics in FAF} demonstrates the SCC-recursive characterization of fuzzy extension semantics.
Section \ref{Illustrating Example} uses examples to illustrate the SCC-recursive schema.
Section \ref{Discussion and Conclusion} discusses related work and concludes the paper.

\section{Preliminaries}\label{Preliminary}

\subsection{Fuzzy Set Theory}

\begin{defn}[\cite{zadeh1965fuzzy}]
A \emph{fuzzy set} is a pair $(X, S)$ in which $X$ is a nonempty set called the \emph{universe} and $S: X\rightarrow[0, 1]$ is the associated \emph{membership function}. For each $x\in X$,  $S(x)$ is called the \emph{grade} of membership of $x$ in $X$. \end{defn}

For convenience, when the universe $X$ is fixed,  a fuzzy set  $(X, S)$ is identified by its membership function $S$,  which can be represented by a set of pairs $(x,a)$ with $x\in X$ and $a\in[0,1]$.
We stipulate that all pairs $(x,0)$ are omitted from $S$.
For any $X'\subseteq X$, we denote by $S|_{X'}$ the \emph{restriction} of $S$ to $X'$: for any $x\in X'$, $S|_{X'}(x)=S(x)$, and for any $x\notin X'$, $S|_{X'}(x)=0$.

For instance, the following are fuzzy sets with universe $\{A,B,C\}$: 
$$S_1=\{(A,0.5)\},  S_2=\{(B,0.8),(C,1)\},  S_3=\{(A,0.8),(B,0.8),(C,1)\}.$$
Note that $S_1(A)=0.5, S_1(B)=S_1(C)=S_2(A)=0$, and in $S_3$ every element has a non-zero grade. Evidently, $S_2$ is the restriction of $S_3$ on $\{B,C\}$, i.e., $S_2=S_3|_{\{B,C\}}$.

A \emph{fuzzy point} is a fuzzy set containing a unique pair $(x,a)$. We may identify a fuzzy point  by its pair. For example, $S_1$ is a fuzzy point and identified by $(A,0.5)$.

Let $S_1$ and $S_2$ be two fuzzy sets. Say $S_1$ is a \emph{subset} of $S_2$, denoted by $S_1\subseteq S_2$, if for any $x\in X$, $S_1(x)\le S_2(x)$. Conventionally, we write $(x, a)\in S$ if a fuzzy point  $(x, a)$ is a subset of $S$. Moreover, we shall use the following notations:
\begin{itemize}
\setlength{\itemsep}{0pt}
\setlength{\parsep}{0pt}
\setlength{\parskip}{0pt}
  \item the \emph{union}  of $S_1$ and $S_2$: $S_1\cup S_2=\{(x,\max\{S_1(x),S_2(x)\})\mid x\in X\}$;
  \item the \emph{intersection} of $S_1$ and $S_2$: $S_1\cap S_2=\{(x,\min\{S_1(x),S_2(x)\})\mid x\in X\}$.
\end{itemize}

In the above example,  $S_1(x)\leq S_3(x)$ for each element $x$, thus fuzzy point $S_1$ is a subset of $S_3$, written as $(A,0.5)\in S_3$. Similarly, it is easy to check: (i) $S_2\subseteq S_3$; (ii)  $S_2\cup S_3=\{(A,0.8),(B,0.8),(C,1)\}$; (iii) $S_1\cap S_3=\{(A,0.5)\}$. 

\subsection{Fuzzy Argumentation Frameworks}

In this paper, we focus on \emph{fuzzy argumentation framework} and its associated \emph{fuzzy extension semantics} introduced in \cite{wu2016godel}.

\begin{defn}\label{FAS}
A \emph{fuzzy argumentation framework} (FAF) over a finite set of arguments $Args$ is a pair  $\mathcal{F}=\langle\mathcal{A}, \mathcal{R}\rangle$ in which $\mathcal{A}: Args \rightarrow [0, 1]$ and $\mathcal{R}: Args \times Args \rightarrow [0, 1]$ are total functions. 
\end{defn}

In the definition, $\mathcal{A}$ and  $\mathcal{R}$ are fuzzy sets of arguments and attacks. $\mathcal{A}$ can be denoted by pairs $(A,\mathcal{A}(A))$ 
and $\mathcal{R}$ can be denoted by pairs $\big( (A,B),\mathcal{R}(A,B)\big)$ or simply $\big( (A,B),\mathcal{R}_{_{AB}}\big)$. 
Moreover, we denote by $Att(A)$ the set of all  attackers of $A$, i.e.,  $Att(A)=\{B\in Args\mid  \mathcal{R}(B,A)> 0\}.$ 
For instance, we depict an FAF over $Args=\{A, B, C\}$ in Figure \ref{a simple FAF}, where
\begin{displaymath}
\mathcal{F}=\langle\{(A,0.8),(B,0.7),(C,0.9)\},\{\big( (A,B),0.8\big), \big( (B,C),0.9\big)\}\rangle.
\end{displaymath}
In the subsequent section, we apply a simple representation:
\begin{itemize}
    \item $(A,a)$ can be represented as $A_a$,
    \item $\big( (A,B),r\big)$ can be represented as $A\xrightarrow{r} B$.
\end{itemize}
Therefore, the above FAF can be represented as
\begin{displaymath}
\mathcal{F}=\langle\{A_{0.8},B_{0.7}, C_{0.9}\},\{A\xrightarrow{0.8} B, B\xrightarrow{0.9} C\}\rangle.
\end{displaymath}

\begin{figure}[H]
\centering
\scalebox{0.8}{
\begin{tikzpicture}[
roundnode/.style={circle, draw=black!100, fill=white!5, thick, minimum size=5mm},
rectanglenode/.style={rectangle, draw=black!0, fill=white!5, thick, minimum size=0mm},
]
\node[roundnode]      (1) at(-2,0)      {$A: {0.8}$};
\node[roundnode]      (2) at(1,0)      {$B: {0.7}$};
\node[roundnode]      (3) at(4,0)      {$C: {0.9}$};
\draw[->,thick] (1)--(2);
\draw[->,thick] (2)--(3);
\node at(-0.5,0.2) {0.8};
\node at(2.5,0.2) {0.9};
\end{tikzpicture}
}
\caption{A simple FAF}
\label{a simple FAF}
\end{figure}

While arguments with conflict cannot be accepted together in classical semantics,
semantics in quantitative settings allow for a certain degree of tolerance for internal conflicts among arguments \cite{da2011changing,dunne2011weighted,hunter2013probabilistic}, enabling weak attacks to be ignored. 
We review the notion of \emph{tolerable} attacks from \cite{wu2016godel}.

\begin{defn}
Suppose $(A, a)$ attacks $(B, b)$ w.r.t. $\mathcal{R}_{_{AB}}$. 
Then the attack is \emph{tolerable} if $a*\mathcal{R}_{_{AB}}+b\leq 1$, otherwise it is \emph{sufficient}.
Here, $*$ is a shorthand s.t. $a*\mathcal{R}_{_{AB}}=\min\{a,\mathcal{R}_{_{AB}}\}$.
\end{defn}

Note that the degrees of the attacker and the attack relation are aggregated toward the attackee.
Intuitively, a tolerable attack can be \emph{ignored} and a sufficient attack \emph{weakens} the attackee.

\begin{defn}\label{weaken}
Let $(A, a)$ attacks $(B, b)$ w.r.t. $\mathcal{R}_{_{AB}}$.
Then $(A, a)$ \emph{weakens} $(B, b)$ to $(B, b')$ where $b'= \min\{1 - a*\mathcal{R}_{_{AB}}, b\}$, or precisely
\begin{displaymath}
   b'=
   \begin{cases}
       b                                  & \text{if } a*\mathcal{R}_{_{AB}}+b\leq 1, \\
     1 - a*\mathcal{R}_{_{AB}} & \text{if } a*\mathcal{R}_{_{AB}}+b> 1.
   \end{cases}
\end{displaymath}
We say that $S\subseteq \mathcal{A}$ \emph{weakens} $(B,b)$ to $(B,b')$ if $\exists (C,c)\in S$ weakens $(B,b)$ to $(B,b')$.
\end{defn}

The notion of \emph{weakening defence} and its associated characteristic function are reviewed below, indicating that a fuzzy set of arguments defends a fuzzy argument by weakening its attackers.

\begin{defn}\label{weakening defend}
A fuzzy set $S$ \emph{weakening defends} a fuzzy argument $(A, a)$ iff, for any $(B, b)$ that sufficiently attacks $(A, a)$, $S$ weakens $(B, b)$ to $(B, b^{'})$ s.t. $(B, b^{'})$ tolerably attacks $(A, a)$.
\end{defn}

\begin{defn}
The characteristic function of an FAF $\mathcal{F}=\langle\mathcal{A},\mathcal{R}\rangle$ is a function $F_{\mathcal{F}}$ from the set of all the subsets of $\mathcal{A}$ to itself, such that $\forall S\in \mathcal{A}$, $F_{\mathcal{F}}(S) = \{(A, a) \mid S \text{ weakening defends } (A, a)\}.$
\end{defn}

The fuzzy extension semantics in \cite{wu2016godel} are reviewed as follows.

\begin{defn}
Let $\mathcal{F}=\langle\mathcal{A},\mathcal{R}\rangle$ be an FAF and a fuzzy set $E\subseteq \mathcal{A}$.
$E$ is \emph{conflict-free} if all attacks in $E$ are tolerable.
Let $E$ be a conflict-free fuzzy set. Then we define:

$E$ is an \emph{admissible} fuzzy extension if $E$ weakening defends each element in $E$, i.e., $E\subseteq F_{\mathcal{F}}(E)$.

$E$ is a \emph{complete} fuzzy extension if it contains all the elements in $\mathcal{A}$ that $E$ weakening defends, i.e., $E=F_{\mathcal{F}}(E)$.

$E$ is a \emph{preferred} fuzzy extension iff it is a maximal admissible fuzzy extension.

$E$ is the \emph{grounded} fuzzy extension iff it is the least fixed point of the characterization function $F_{\mathcal{F}}$.

The set of fuzzy extensions under a given semantics $\mathcal{S}$ is denoted by $\mathcal{E}_{\mathcal{S}}(\mathcal{F})$. 
\end{defn}

Whereas classical extension semantics identify the arguments that can be accepted, fuzzy extension semantics quantify the degree to which arguments can be accepted, also called \emph{acceptability degree} in this paper.

\begin{exmp}
Continue considering the FAF depicted in Figure \ref{a simple FAF}.
We compute a complete fuzzy extension $E$. 
Given that $A$ has no attackers, we have $E(A)=\mathcal{A}(A)=0.8$.
Since $B$ is weakened by $A$, it follows that $E(B)=1-E(A)*\mathcal{R}_{_{AB}}=0.2$.
As $B$ is weakened, $C$ is weakening defended to the degree of $0.8$, leading to $E(C)=0.8$.
Therefore, the acceptability degrees of $A$, $B$ and $C$ are $0.8$, $0.2$ and $0.8$, respectively.
Consequently, we obtain a complete fuzzy extension $\{(A,0.8),(B,0.2),(C,0.8)\}$.
\end{exmp}

\section{SCC-recursiveness}\label{SCC-recursive in FAF}
\subsection{Graph Notations}

The notion of SCC-recursiveness relies on the graph-theoretical notion of strongly connected components.
To begin, we should integrate the graph notations into FAF.

\begin{defn}
Let $\mathcal{F}=\langle\mathcal{A}, \mathcal{R}\rangle$ be an FAF over $Args$. 
Given an argument $A\in Args$, we define the parent nodes of $A$ as $par_{\mathcal{F}}(A)=\{B\mid \mathcal{R}(B,A)> 0\}$.
$A$ is called an \emph{initial} node if $par_{\mathcal{F}}(A)=\varnothing$.
\end{defn}

\begin{defn}
Let $\mathcal{F}=\langle\mathcal{A},\mathcal{R}\rangle$ be an FAF over $Args$, $A\in Args$ and $S,P\subseteq \mathcal{A}$. We define that:
\begin{itemize}
\setlength{\itemsep}{0pt}
\setlength{\parsep}{0pt}
\setlength{\parskip}{0pt}
    \item $S$ attacks $A$ iff $\exists B\in S$ s.t. $\mathcal{R}(B,A)>0$;
    \item $A$ attacks $S$ iff $\exists B\in S$ s.t. $\mathcal{R}(A,B)>0$;
    \item  $S$ attacks $P$ iff $\exists A\in S$  and  $\exists B\in P$ s.t. $\mathcal{R}(A,B)>0$;
    \item  $outpar_{\mathcal{F}}(S)=\{A\in Args \mid A\notin S \text{ and } A\text{ attacks }S\}$.
\end{itemize}
\end{defn}

The notions of \emph{path} and \emph{path-equivalence} are defined as follows.

\begin{defn}
Let $\mathcal{F}=\langle\mathcal{A},\mathcal{R}\rangle$ be an $\mathit{FAF}$. We say that there is a \emph{path} from $A_1$ to $A_n$ iff there is a sequence $\{A_1,A_2,...,A_n\}$ such that $\mathcal{R}(A_i,A_{i+1})>0$ for $i\in \{1,...,n-1\}$.

\end{defn}

\begin{defn}
Let $\mathcal{F}=\langle\mathcal{A},\mathcal{R}\rangle$ be an FAF over $Args$. The binary relation of \emph{path-equivalence} between nodes, denoted as $PE_{\mathcal{F}}\subseteq Args\times Args$, is defined as follows:
\begin{itemize}
    \item $\forall A\in Args$, $(A,A)\in PE_{\mathcal{F}}$;
    \item given two distinct arguments $A,B\in Args$, $(A,B)\in PE_{\mathcal{F}}$ iff there is a path from $A$ to $B$ and a path from $B$ to $A$.
\end{itemize} 
\end{defn}

The notion of \emph{strongly connected components} is defined as follows.

\begin{defn}
Let $\mathcal{F}=\langle\mathcal{A},\mathcal{R}\rangle$ be an FAF over $Args$.
The equivalence classes under the path-equivalence relation are called \emph{strongly connected components} (SCCs) of $\mathcal{F}$.
We denote the set of SCCs of $\mathcal{F}$ by $SCCS_{\mathcal{F}}$.
Given an argument $A\in Args$, the SCC that $A$ belongs to is denoted as $SCC_{\mathcal{F}}(A)=\{B \mid (A,B)\in PE_{\mathcal{F}}\}$.
\end{defn}

We extend the notion of parents to SCCs, representing the set of other SCCs that attack a given SCC $S$ as $sccpar_{\mathcal{F}}(S)$.
Additionally, we introduce the concept of proper ancestors, denoted as $sccanc_{\mathcal{F}}(S)$.

\begin{defn}
Let $\mathcal{F}=\langle\mathcal{A},\mathcal{R}\rangle$ be an FAF and  $S\in SCCS_{\mathcal{F}}$. 
We define
\begin{displaymath}
\begin{aligned}
sccpar_{\mathcal{F}}(S)& =\{P\in SCCS_{\mathcal{F}}\mid P\neq S \text{ and P attacks S}\} \\
sccanc_{\mathcal{F}}(S)& =sccpar_{\mathcal{F}}(S)\cup \bigcup\limits_{P\in sccpar_{\mathcal{F}}(S)} sccanc_{\mathcal{F}}(P)
\end{aligned}
\end{displaymath}
An SCC $S$ is called \emph{initial} if $sccpar_{\mathcal{F}}(S)=\varnothing$.
\end{defn}

For the purpose of decomposition, we introduce the notion of \emph{restriction} of an FAF.
Before the formal definition, consider attacks that are \emph{always} tolerable.
For instance, given an FAF $\mathcal{F}: A_{0.3}\xrightarrow{1.0} B_{0.2}$, it is evident that $\mathcal{A}(A)*R(A,B)+\mathcal{A}(B)=0.3*1+0.2\leq 1$, indicating that the attack is always tolerable in $\mathcal{F}$. 
We eliminate such always tolerable attacks when obtaining the restricted sub-frameworks from an FAF.

\begin{defn}\label{restriction of FAF}
Let $\mathcal{F}=\langle\mathcal{A},\mathcal{R}\rangle$ be an FAF over $Args$ and $\mathcal{A}'\subseteq \mathcal{A}$. The \emph{restriction} of $\mathcal{F}$ to $\mathcal{A}'$ is the sub-framework
${\mathcal{F}\!\downarrow}_{\mathcal{A}'}=\langle \mathcal{A}',\mathcal{R}'\rangle$ where $\mathcal{R}'$ satisfies that
\begin{itemize}
    \item if $\mathcal{A}'(A)*\mathcal{R}(A,B)+\mathcal{A}'(B)>1$, then $\mathcal{R}'(A,B)=\mathcal{R}(A,B)$;
    \item otherwise $\mathcal{R}'(A,B)=0$.
\end{itemize}
\end{defn}

For simplicity of discussion, we assume that the original FAF contains no always tolerable attacks in the subsequent discussion.

\subsection{Basic Theory}

While the idea behind SCC-recursiveness is intuitive and natural, the initial formalization of its required notions may seem quite complex due to its recursive nature.
Therefore, we clarify its basic idea in this section.\footnote{Referring to the examples provided in Section \ref{Illustrating Example} helps in understanding the concepts in this section.}
In the following, we consider a generic FAF $\mathcal{F}=\langle\mathcal{A},\mathcal{R}\rangle$ and a semantics $\mathcal{S}\in$ \textit{\{admissible, complete, preferred, grounded\}}.

To start, let us treat SCCs as single nodes.
Then any FAF can be viewed as a directed acyclic graph; that is, the attack relation induces a partial order among the SCCs.
Furthermore, it is evident that the acceptability degree of an argument depends on its ancestor nodes.
This implies that semantics can be computed following the sequence of SCCs.
We begin by computing semantics for an initial SCC $\hat{S}$.
To achieve this, we examine the sub-framework over $\hat{S}$ by restricting $\mathcal{F}$ to $\mathcal{A}|_{\hat{S}}$.
The semantics of this sub-framework are processed by a \emph{base function}, denoted as $\mathcal{BF}_\mathcal{S}$, which is defined to return the set of all fuzzy extensions under semantics $\mathcal{S}$.


Now, we arrive at the crucial problem: how to compute semantics for a given SCC $S$ after computing its ancestor SCCs.
Suppose $A\in S$ and $E\in \mathcal{E}_{\mathcal{S}}(\mathcal{F})$.
Let $\max\limits_{B\in {outpar_{\mathcal{F}}(S)}} E(B)*\mathcal{R}_{_{BA}}=\tilde{a}$.
Then $A$ is weakened to the lesser of $1-\tilde{a}$ or $\mathcal{A}(A)$.
Following Definition \ref{weakening defend}, only $A$'s unweakened degree can influence its target arguments within $S$.
This implies that we only need to consider a restricted sub-framework over $S$ where arguments with the unweakened degree, e.g., the lesser of $1-\tilde{a}$ or $\mathcal{A}(A)$ for $A$. 
Note that the relevant attacks may be suppressed when obtaining the restricted sub-framework, leading to the recursive decomposition of SCCs.
Furthermore, an argument can be accepted to some degree only if it can be defended to that degree.
Therefore, it is necessary to identify the degree to which an argument can be defended by $E$ from outside $S$.
Consequently, we can distinguish three components:
\begin{itemize}
    \item  \emph{Weakened Component}, denoted as $W_{\mathcal{F}} (S,E)$, which represents the weakened degree of arguments in $S$ by the ancestor SCCs.
    \item \emph{Restricted Component}, denoted as $R_{\mathcal{F}} (S,E)$, which represents the unweakened degree of arguments in $S$.
    \item \emph{Defended Component}, denoted as $D_{\mathcal{F}} (S,E)$, a subset of $R_{\mathcal{F}} (S,E)$ that represents the degree to which an argument can be defended by $E$ from outside $S$.
\end{itemize}

\begin{defn}\label{WRD Component}
Let $\mathcal{F}=\langle\mathcal{A},\mathcal{R}\rangle$ be an FAF, $E\subseteq \mathcal{A}$ and $S\in SCCS_{_{\mathcal{F}}}$. We define that
    \begin{align*}
    \mathit{W}_{\mathcal{F}} (S,E)=&\{(A,a)\mid A\in S, a=\max \limits_{B \in outpar_{\mathcal{F}}(S)} E(B)*\mathcal{R}_{_{BA}}\}\\
    \mathit{R}_{\mathcal{F}} (S,E)=&\{(A,a)\mid A\in S, a=\min\{1-W_{\mathcal{F}} (S,E)(A), \mathcal{A}(A)\}\}\\
    \mathit{D}_{\mathcal{F}} (S,E)=&\{(A,a)\mid A\in S, \forall B\in outpar_{\mathcal{F}}(S), E \text{ weakens } (B,\mathcal{A}(B)) \text{ to }\\
    &   (B,b) \text{ that tolerably attacks } (A,a)\}.
    \end{align*}
\end{defn}
 
From the above discussion, the computation of semantics over $S$ depends on the restricted sub-framework $\mathcal{F}\!\downarrow _{R_{\mathcal{F}}(S,E)}$ and $\mathit{D}_{\mathcal{F}} (S,E)$.
For the purpose of computation, we define a \emph{generic function}, denoted as $\mathcal{GF}$, where
      \begin{itemize}
          \item input: a (possibly restricted) FAF $\mathcal{F}=\langle \mathcal{A},\mathcal{R}\rangle$ and a fuzzy set $C$;
          \item output: a set of subsets of $C$.
      \end{itemize}

We will use the notation $\mathcal{GF}(\mathcal{F},C)$ for the generic function, which is defined as follows.
If $\mathcal{F}$ consists of exactly one SCC, then $\mathcal{GF}(\mathcal{F},C)$ coincides with a \emph{base function} $\mathcal{BF}_{\mathcal{S}}(\mathcal{F},C)$, which is defined to obtain the fuzzy extensions of $\mathcal{F}$ contained in $C$ under semantics $\mathcal{S}$.\footnote{Note that this base function covers the previous case for computing the semantics of initial SCCs.}
On the other hand, if $\mathcal{F}$ can be decomposed into several SCCs, then $\mathcal{GF}(\mathcal{F},C)$ is obtained by recursively applying $\mathcal{GF}$ to each SCC of $\mathcal{F}$. 
Formally, this means that for any $S\in SCCS_{\mathcal{F}}$, $E|_S\in \mathcal{GF}(\mathcal{F}\!\downarrow_{R_{\mathcal{F}}(S,E)},C')$, where $C'$ represents the defended component.
Note that $C'$ is determined by considering both the attacks coming from outside $\mathcal{F}$ (as $\mathcal{F}$ is possibly restricted) and those coming from other SCCs within $\mathcal{F}$, yielding $C'=C\cap D_{\mathcal{F}}(S,E)$.

We now formally introduce SCC-recursiveness as a principle for fuzzy extension semantics.

\begin{defn}\label{The base definition of SCC-recursiveness in FAF}
A semantics $\mathcal{S}$ is \emph{SCC-recursive} iff for any FAF $\mathcal{F}=\langle\mathcal{A},\mathcal{R}\rangle$, $\mathcal{E}_\mathcal{S}(\mathcal{F})=\mathcal{GF}(\mathcal{F},\mathcal{A})$, where for any $\mathcal{F}=\langle\mathcal{A},\mathcal{R}\rangle$ and for any $C\subseteq \mathcal{A}$, the function $\mathcal{GF}(\mathcal{F},C)\subseteq 2^\mathcal{A}$ is defined as follows:
for any $E\subseteq \mathcal{A}$, $E\in \mathcal{GF}(\mathcal{F},C)$ if and only if
\begin{itemize}
\item in case $|SCCS_{\mathcal{F}}|=1$, $E\in \mathcal{BF}_\mathcal{S}(\mathcal{F}, C)$,
\item otherwise, $\forall S\in SCCS_{\mathcal{F}}$, $E|_S\in \mathcal{GF}(\mathcal{F}\!\downarrow_{R_{\mathcal{F}}(S,E)},D_{\mathcal{F}}(S,E)\cap C)$,
\end{itemize}
where $\mathcal{BF}_\mathcal{S}(\mathcal{F}, C)$ is the \emph{base function} that, given an FAF $\mathcal{F}=\langle\mathcal{A},\mathcal{R}\rangle$ s.t. $|SCCS_{\mathcal{F}}|=1$ and a fuzzy set $C\subseteq \mathcal{A}$, returns a set of fuzzy extensions contained in $C$.
\end{defn}

As noticed before, the generic function $\mathcal{GF}(\mathcal{F}, C)$ is recursively defined.
The base of the recursion is given by the base function $\mathcal{BF}_{\mathcal{S}}(\mathcal{F}, C)$, which returns the results for $\mathcal{F}$ consisting of a single SCC. 
When $\mathcal{F}$ consists of more than one SCC, the recursive step involves a decomposition schema along its SCCs.
Consequently, to characterize a semantics as SCC-recursive, we need to identify its base function and demonstrate that it fits the decomposition schema.

The definition naturally provides a schema for computing SCC-recursive semantics. 
Consider a generic FAF $\mathcal{F}=\langle \mathcal{A},\mathcal{R}\rangle$ over $Args$. 
First, for any initial SCC $S$, we have $R_{\mathcal{F}}(S,E) = D_{\mathcal{F}}(S,E)=\mathcal{A}|_{S}$.
Thus, the restricted sub-framework over $S$ is $\mathcal{F}\!\downarrow_{\mathcal{A}|_{S}}$, which clearly consists of a unique SCC.
According to Definition \ref{The base definition of SCC-recursiveness in FAF}, the base function $\mathcal{BF}_\mathcal{S}(\mathcal{F}\!\downarrow_{\mathcal{A}|_{S}},\mathcal{A}|_{S})$ is invoked, returning the set of fuzzy extensions of $\mathcal{F}\!\downarrow_{\mathcal{A}|_{S}}$ under semantics $\mathcal{S}$.
The results are then utilized to identify the restricted sub-frameworks in subsequent SCCs, taking account of the restricted and defended components.
This procedure is recursively invoked and can be summarized as follows:
\begin{enumerate}
\setlength{\itemsep}{0pt}
\setlength{\parsep}{0pt}
\setlength{\parskip}{0pt}
  \item A (possibly restricted) FAF is partitioned into its SCCs; they form a partial order induced by the attack relation.
  \item The set of fuzzy extensions over each initial SCC is determined using a semantic-specific base function.
  \item For each fuzzy extension determined at step 2, the restricted and defended components within subsequent SCCs are identified; then the associated restricted sub-framework is taken into account.
  \item The steps 1–3 are applied recursively on the restricted FAF obtained at step 3.
\end{enumerate}

\section{SCC-recursive Characterization of Semantics}\label{SCC-recursive semantics in FAF}

\subsection{Generalized Fuzzy Extension Semantics}

In order to develop an SCC-recursive characterization of semantics, it is necessary to redefine fuzzy extension semantics with reference to a specific subset $C$.
In the following, we consider a generic FAF $\mathcal{F} =\langle \mathcal{A},\mathcal{R}\rangle$ and a fuzzy set $C\subseteq \mathcal{A}$.
The notion of admissible fuzzy extension in $C$ is defined as follows.

\begin{defn}\label{the admissible extension in C}
A fuzzy set $E\subseteq \mathcal{A}$ is an admissible fuzzy extension in $C$ iff $E\subseteq C$ and $E\in \mathcal{AE}(\mathcal{F})$.
The set of admissible fuzzy extensions in $C$ is denoted as $\mathcal{AE}(\mathcal{F},C)$.
\end{defn}

Next, we introduce the notions of complete and preferred fuzzy extensions in $C$.

\begin{defn}\label{the complete extension in C}
A fuzzy set $E$ is a complete fuzzy extension in $C$ iff $E\in \mathcal{AE}(\mathcal{F},C)$,
and it contains all the elements in $C$ that $E$ weakening defend.
The set of complete fuzzy extensions in $C$ is denoted as $\mathcal{CE}(\mathcal{F},C)$.
\end{defn}

\begin{defn}\label{the preferred extension in C}
A preferred fuzzy extension in $C$ is a maximal element of $\mathcal{AE}(\mathcal{F},C)$.
The set of preferred fuzzy extensions in $C$ is denoted as $\mathcal{PE}(\mathcal{F},C)$.
\end{defn}

The following proposition shows that preferred fuzzy extensions always exist for any FAF $\mathcal{F}=\langle \mathcal{A},\mathcal{R}\rangle$ and for any $C\subseteq \mathcal{A}$.

\begin{prop}
\label{the existense of preferred extensions}
Given an FAF $\mathcal{F}=\langle \mathcal{A},\mathcal{R}\rangle$ and a fuzzy set $C\subseteq \mathcal{A}$, there is always a preferred fuzzy extension $E\in \mathcal{PE}(\mathcal{F},C)$.
\end{prop}

\begin{proof}
We first prove that for any $S\in \mathcal{AE}(\mathcal{F},C)$, if $(A,a)\in C$ is weakening defended by $S$, then $S\cup \{(A,a)\}$ is also admissible in $C$.
Clearly, $S\cup \{(A,a)\}\subseteq C$ and $S\cup \{(A,a)\}\subseteq F_{(\mathcal{F},C)}(S)$.
So we only need to prove that $S\cup \{(A,a)\}$ is conflict-free.
Let's prove it by contradiction.
Suppose that $S\cup \{(A,a)\}$ is not conflict-free.
Then by the hypothesis that $S$ is admissible, $\exists (B,b)\in S$ s.t. $(B,b)$ sufficiently attacks $(A,a)$ or $(A,a)$ sufficiently attacks $(B,b)$.
Since $(A,a)$ is weakening defended by $S$, if $(B,b)$ sufficiently attacks $(A,a)$, then $S$ weakens $(B,b)$ to $(B,b')$ which tolerably attacks $(A,a)$.
It contradicts that $(B,b)\in S$ and $S$ is conflict-free.
Moreover, if $(A,a)$ sufficiently attacks $(B,b)$, then $S$ weakens $(A,a)$ to $(A,a')$ s.t. $(A,a')$ tolerably attacks $(B,b)$, which again leads to that $\exists (C,c)\in S$ sufficiently attacks $(A,a)$.
Consequently, for any fuzzy set $S\in \mathcal{AE}(\mathcal{F},C)$, if $(A,a)\in C$ is weakening defended by $S$, then $S\cup \{(A,a)\}$ is also admissible in $C$.
It directly follows that the elements in $\mathcal{AE}(\mathcal{F},C)$ form a complete partial order w.r.t. fuzzy set inclusion.
Since $\{\varnothing\}$ is always an admissible fuzzy extension in $C$, there is always a preferred fuzzy extension $E\in \mathcal{PE}(\mathcal{F},C)$.
\end{proof}

Proposition \ref{Preferred} shows that a preferred fuzzy extension in $C$ is also complete in $C$.

\begin{prop}\label{Preferred}
A preferred fuzzy extension in $C$ is also complete in $C$.
\end{prop}

\begin{proof}
From the proof of Proposition \ref{the existense of preferred extensions}, it is evident that a preferred fuzzy extension in $C$ is always complete in $C$.
\end{proof}

\begin{defn}
The characteristic function of $\mathcal{F}$ in $C$ is defined as follows: 
\begin{itemize}
  \item $F_{(\mathcal{F},C)}:2^C\rightarrow 2^C$
  \item $F_{(\mathcal{F},C)}(S)=\{(A,a)\mid (A,a)\in C, (A,a) \text{ is weakening defended by } S\}.$
\end{itemize}
\end{defn}

It is easy to see that $F_{(\mathcal{F},C)}$ is monotonic w.r.t. fuzzy set inclusion.
Then the grounded fuzzy extension in $C$ can be defined in terms of the least fixed point of the characteristic function in $C$.

\begin{defn}\label{the grounded extension in C}
The grounded fuzzy extension in $C$, denoted as $GE(\mathcal{F},C)$, is the least fixed point of $F_{(\mathcal{F},C)}$.
\end{defn}

The following proposition demonstrates that the grounded fuzzy extension always exists and is unique for any FAF $\mathcal{F}=\langle \mathcal{A},\mathcal{R}\rangle$ and any $C\subseteq \mathcal{A}$.

\begin{prop}\label{Grounded Lemma}
For any FAF $\mathcal{F} =\langle \mathcal{A},\mathcal{R}\rangle$ and any $C\subseteq \mathcal{A}$, $GE(\mathcal{F},C)$ exists and is unique.
\end{prop}

\begin{proof}
From that $F_{(\mathcal{F},C)}$ is monotonic, it follows from the Knaster-Tarski theorem that $F_{(\mathcal{F},C)}$ has the least fixed point.
Therefore, $GE(\mathcal{F},C)$ exists and is unique.
\end{proof}

Proposition \ref{Grounded} states that the grounded fuzzy extension in $C$ is also the least complete fuzzy extension in $C$.

\begin{prop}\label{Grounded}
$GE(\mathcal{F},C)$ is the least complete fuzzy extension in $C$.
\end{prop}

\begin{proof}
Suppose an FAF $\mathcal{F} =\langle \mathcal{A},\mathcal{R}\rangle$ over a finite set of arguments $Args$ and $C\subseteq \mathcal{A}$.
Since $Args$ is finite, we can construct the least complete fuzzy extension by letting $F^i_{(\mathcal{F},C)}(\varnothing)$ for ordinals $i$, where $F^1_{(\mathcal{F},C)}(\varnothing)=F_{(\mathcal{F},C)}(\varnothing)$ and $F^i_{(\mathcal{F},C)}(\varnothing)=F_{(\mathcal{F},C)}(F^{i-1}_{(\mathcal{F},C)}(\varnothing))$.
Then $\bigcup_{i=1,...,\infty} F^i_{(\mathcal{F},C)}(\varnothing)$ is clearly a complete fuzzy extension in $C$.
We now prove $GE(\mathcal{F},C)=\bigcup_{i=1,...,\infty} F^i_{(\mathcal{F},C)}(\varnothing)$:
\begin{itemize}
    \item Given that $\bigcup_{i=1,...,\infty} F^i_{(\mathcal{F},C)}(\varnothing)$ is clearly a fixed point of $F_{(\mathcal{F},C)}$ and  $GE(\mathcal{F},C)$ is the least fixed point, we have $GE(\mathcal{F},C)\subseteq \bigcup_{i=1,...,\infty} F^i_{(\mathcal{F},C)}(\varnothing)$.
    \item From the construction of $F^i_{(\mathcal{F},C)}(\varnothing)$, it is easy to see that $F^i_{(\mathcal{F},C)}(\varnothing)$ should be contained in any fixed point of $F_{(\mathcal{F},C)}$, and therefore \\$\bigcup_{i=1,...,\infty} F^i_{(\mathcal{F},C)}(\varnothing)\subseteq GE(\mathcal{F},C)$.
\end{itemize} 
Consequently, $GE(\mathcal{F},C)$ is the least complete fuzzy extension in $C$.
\end{proof}

Since the original version of fuzzy extension semantics is recovered by letting $C = \mathcal{A}$, the generalized definition covers the original ones.

\subsection{Admissible Semantics}

We first establish that admissible semantics fit the decomposition schema along SCCs.
The characterization serves as the foundation for analyzing other semantics. 
This is achieved by Theorem \ref{admissible proposition}, which requires two preliminary lemmas.

\begin{lem}\label{the lemma for the right arrow of admissible}
Let $\mathcal{F}=\langle\mathcal{A},\mathcal{R}\rangle$ be an FAF  and $E$ be an admissible fuzzy extension.
Suppose $(A,a)\in F_{\mathcal{F}}(E)$, denoting $SCC_{\mathcal{F}}(A)$ as S, then it holds that:

\begin{itemize}
  \item $(A,a)\in D_{\mathcal{F}}(S,E)$;
  \item $(A,a)$ is weakening defended by $E|_S$ in $\mathcal{F}\!\downarrow_{R_{\mathcal{F}}(S,E)}$.
\end{itemize}
\end{lem}
\begin{proof}

From $(A,a)\in F_{\mathcal{F}}(E)$, it follows that for any $B\in outpar_{\mathcal{F}}(S)$, if $(B,\mathcal{A}(B))$ sufficiently attacks $(A,a)$, then $E$ weakens $(B,b)$ to $(B,b')$ which tolerably attacks $(A,a)$, i.e., $(A,a)\in D_{\mathcal{F}}(S,E)$.

Turning to the second part of the proof, since $E$ is admissible, all the elements in $E$ are weakening defended by $E$.
Thus, $E|_S\subseteq F_\mathcal{F}(E)$ and therefore $E|_S\subseteq D_\mathcal{F}(S,E)$.
Now we need to prove that $(A,a)$ is weakening defended by $E|_S$ in $\mathcal{F}\!\downarrow_{R_{\mathcal{F}}(S,E)}$.
Assume $(B,b)\in R_\mathcal{F}(S,E)$ that sufficiently attacks $(A,a)$ in $\mathcal{F}\!\downarrow_{R_{\mathcal{F}}(S,E)}$.
Then from Definition \ref{restriction of FAF}, $(B,b)$ also sufficiently attacks $(A,a)$ in $\mathcal{F}$.
Since $(A,a)\in F_\mathcal{F}(E)$, $E$ weakens $(B,b)$ to $(B,b')$ which tolerably attacks $(A,a)$. 
Moreover, $(B,b)$ is weakened to $(B,b')$ by $E|_S$ in $\mathcal{F}$ due to $(B,b)\in R_\mathcal{F}(S,E)$.
Again utilizing Definition \ref{restriction of FAF}, $(B,b)$ is weakened to $(B,b')$ by $E|_S$ in $\mathcal{F}\!\downarrow_{R_{\mathcal{F}}(S,E)}$.
Therefore, $E|_S$ weakens $(B,b)$ to $(B,b')$ s.t. $(B,b')$ tolerably attacks $(A,a)$ in $\mathcal{F}\!\downarrow_{R_{\mathcal{F}}(S,E)}$.
Consequently, $(A,a)$ is weakening defended by $E|_S$ in $\mathcal{F}\!\downarrow_{R_{\mathcal{F}}(S,E)}$.
\end{proof}

\begin{lem}\label{the lemma for the left arrow of admissible}
Given an FAF $\mathcal{F}=\langle\mathcal{A},\mathcal{R}\rangle$, let $E\subseteq \mathcal{A}$ be a fuzzy set s.t.  $\forall S\in SCCS_{\mathcal{F}}$, $E|_S\in \mathcal{AE}(\mathcal{F}\!\downarrow_{R_{\mathcal{F}}(S,E)}, D_{\mathcal{F}}(S,E))$.
Then for any $\hat{S}\in SCCS_\mathcal{F}$ and any $(A,a)\in D_\mathcal{F}(\hat{S},E)$, if $(A,a)$ is weakening defended by $E|_{\hat{S}}$ in $\mathcal{F}\!\downarrow_{R_{\mathcal{F}}(\hat{S},E)}$, 
then $(A,a)$ is weakening defended by $E$ in $\mathcal{F}$. 
\end{lem}

\begin{proof}
We need to prove that for any $(B,b)\in \mathcal{A}$ that sufficiently attacks $(A,a)$,
then $E$ weakens $(B,b)$ to $(B,b')$ which tolerably attacks $(A,a)$.
We distinguish two cases for $B$.

Case 1: $B\in \hat{S}$.
If $(B,b)\notin R_{\mathcal{F}}(\hat{S},E)$, then from Definition \ref{WRD Component}, $E$ weakens $(B,b)$ to $(B,b')$ s.t. $(B,b')\in R_\mathcal{F}(\hat{S},E)$.
Therefore, without loss of generality, suppose $(B,b)\in R_{\mathcal{F}}(\hat{S},E)$.
By the hypothesis of $(A,a)\in F_{\mathcal{F}\!\downarrow_{R_{\mathcal{F}}(\hat{S},E)}}(E|_{\hat{S}})$, it can be concluded that $E$ weakens $(B,b)$ to $(B,b')$ which tolerably attacks $(A,a)$ in $\mathcal{F}\!\downarrow_{R_{\mathcal{F}}(\hat{S},E)}$.
Then according to Definition \ref{restriction of FAF}, $(B,b)$ is weakened to $(B,b')$ by $E$ in $\mathcal{F}$, and $(B,b')$ tolerably attacks $(A,a)$ in $\mathcal{F}$.

Case 2: $B\notin \hat{S}$.
Then $B\in (outpar_\mathcal{F}(\hat{S})\cap par_{\mathcal{F}}(A))$.
By the hypothesis of $(A,a)\in D_\mathcal{F}(\hat{S},E)$, we derive that $E$ weakens $(B,b)$ to $(B,b')$ which tolerably attacks $(A,a)$ in $\mathcal{F}$.

Consequently, $(A,a)$ is weakening defended by $E$ in $\mathcal{F}$.
\end{proof}

\begin{thm}\label{admissible proposition}
Given an FAF $\mathcal{F}=\langle\mathcal{A},\mathcal{R}\rangle$ and a fuzzy set $E\subseteq \mathcal{A}$, it holds that: $\forall C\subseteq \mathcal{A}$, $E\in \mathcal{AE}(\mathcal{F}, C)$ if and only if $\forall S\in SCCS_{\mathcal{F}}$,
\begin{displaymath}
E|_S\in \mathcal{AE}(\mathcal{F}\!\downarrow_{R_{\mathcal{F}}(S,E)},D_{\mathcal{F}}(S,E)\cap C).
\end{displaymath}

\end{thm}
\begin{proof}
$\Rightarrow$: Suppose $E\in \mathcal{AE}(\mathcal{F}, C)$.
Then according to Definition \ref{the admissible extension in C}, $E\subseteq C$ and $E\subseteq F_{\mathcal{F}}(E)$. 
From Lemma \ref{the lemma for the right arrow of admissible}, for any $S\in SCCS_{\mathcal{F}}$ and for any $(A,a)\in E|_S$,
we conclude that $(A,a)\in D_{\mathcal{F}}(S,E)$ and thus $E|_{S}\in D_{\mathcal{F}}(S,E)\cap C$.
Moreover, from the same lemma, $(A,a)$ is weakening defended by $E|_S$ in $\mathcal{F}\!\downarrow_{R_{\mathcal{F}}(S,E)}$.
Since $E$ is admissible and thus clearly conflict-free in $\mathcal{F}$, it entails that $E|_S$ is conflict-free by Definition \ref{restriction of FAF}.
Therefore, $E|_S$ is admissible in ${\mathcal{F}\!\downarrow}_{R_{\mathcal{F}}(S,E)}$.
Consequently, $E|_S\in \mathcal{AE}(\mathcal{F}\!\downarrow_{R_{\mathcal{F}}(S,E)},D_{\mathcal{F}}(S,E)\cap C)$.

$\Leftarrow$: Suppose $\forall S\in SCCS_{\mathcal{F}}$, $E|_S\in \mathcal{AE}(\mathcal{F}\!\downarrow_{R_{\mathcal{F}}(S,E)},D_{\mathcal{F}}(S,E)\cap C)$.
Then $E|_S\subseteq (D_{\mathcal{F}}(S,E)\cap C)$, and thus $E\subseteq C$. 
Now we need to prove that $E$ is admissible in $\mathcal{F}$, i.e., $E$ is conflict-free and $E\subseteq F_\mathcal{F}(E)$.
If $E$ is not conflict-free, then there exists $(A,a),(B,b)\in E$ s.t. $(B,b)$ sufficiently attacks $(A,a)$. We denote $SCC_{\mathcal{F}}(A)$ as $S$. 
It is evident that $B\notin S$; if not, it would contradict the fact that $E|_S$ is admissible and necessarily conflict-free.
So $B\in outpar_\mathcal{F}(S)$ and let $SCC_{\mathcal{F}}(B)$ be denoted as $S'$.
Since $(A,a)\in D_{\mathcal{F}}(S,E)$, it follows that $(B,b)$ is weakened by $E$.
Namely, there exists $(C,c)\in E$ that sufficiently attacks $(B,b)$.
We can again conclude that $C\in outpar_{\mathcal{F}}(S')$, and let $SCC_{\mathcal{F}}(C)$ be denoted as $S''$.
Similarly, from the fact that $(B,b)\in E|_{S'}\in D_\mathcal{F}(S',E)$,
it follows that there exists $(D,d)\in E$ that sufficiently attacks $(C,c)$ and $D\in outpar_{\mathcal{F}}(S'')$.
Since this process would be infinitely repeated yet $SCCS_{\mathcal{F}}$ is finite and acyclic, a contradiction arises.
Therefore, $E$ is conflict-free.
Moreover, consider a generic element $(A,a)\in E$ and let $SCC_{\mathcal{F}}(A)$ be denoted as $\hat{S}$.
Then evidently $(A,a)\in E|_{\hat{S}}$.
From the hypothesis of $E|_{\hat{S}}\in \mathcal{AE}(\mathcal{F}\!\downarrow_{R_{\mathcal{F}}({\hat{S}},E)},D_{\mathcal{F}}({\hat{S}},E)\cap C)$, it follows that $(A,a)\in D_{\mathcal{F}}({\hat{S}},E)$ and $(A,a)\in F_{\mathcal{F}\!\downarrow_{R_{\mathcal{F}}(\hat{S},E)}}(E|_{\hat{S}})$.
Then applying Lemma \ref{the lemma for the left arrow of admissible}, it entails that $(A,a)$ is weakening defended by $E$ in $\mathcal{F}$, i.e., $(A,a)\in F_\mathcal{F}(E)$.
Consequently, $E\in \mathcal{AE}(\mathcal{F},C)$.
\end{proof}

\subsection{Complete Semantics}

The following theorem shows that complete semantics also fit the decomposition schema.

\begin{thm}\label{Complete Proposition}
Given an FAF $\mathcal{F}=\langle\mathcal{A},\mathcal{R}\rangle$ and a fuzzy set $E\subseteq \mathcal{A}$,
it holds that: $\forall C\subseteq \mathcal{A}$, $E\in \mathcal{CE}(\mathcal{F}, C)$ if and only if $\forall S\in SCCS_{\mathcal{F}}$,
\begin{displaymath}
E|_S\in \mathcal{CE}(\mathcal{F}\!\downarrow_{R_{\mathcal{F}}(S,E)},D_{\mathcal{F}}(S,E)\cap C).
\end{displaymath}
\end{thm}

\begin{proof}
$\Rightarrow$: Suppose $E\in \mathcal{CE}(\mathcal{F}, C)$. Then $E\in \mathcal{AE}(\mathcal{F}, C)$, and thus Theorem \ref{admissible proposition} entails that $\forall S\in SCCS_\mathcal{F}$,  $E|_S\in \mathcal{AE}(\mathcal{F}\!\downarrow_{R_{\mathcal{F}}(S,E)},D_{\mathcal{F}}(S,E)\cap C)$.
As a consequence, we need to prove that for any $(A,a)\in (D_\mathcal{F}(S,E)\cap C)$, if $(A,a)$ is weakening defended by $E|_S$ in $\mathcal{F}\!\downarrow_{R_{\mathcal{F}}(S,E)}$, then $(A,a)\in E|_S$.
It is evident that Lemma \ref{the lemma for the left arrow of admissible} can be applied to $(A,a)$, leading to $(A,a)\in F_\mathcal{F}(E)$.
Moreover, $(A,a)\in (D_\mathcal{F}(S,E)\cap C)$ implies that $(A,a)\in C$.
Consequently, from the hypothesis that $E\in \mathcal{CE}(\mathcal{F},C)$, it follows that $(A,a)\in E$, and therefore $(A,a)\in E|_S$.

$\Leftarrow$: Suppose that $\forall S\in SCCS_{\mathcal{F}}$, 
$E|_S\in \mathcal{CE}(\mathcal{F}\!\downarrow_{R_{\mathcal{F}}(S,E)},D_{\mathcal{F}}(S,E)\cap C)$.
Then Theorem \ref{admissible proposition} entails that $E\in \mathcal{AE}(\mathcal{F},C)$.
Therefore, we only need to prove that for any $(A,a)\in C$, if $(A,a)$ is weakening defended by $E$ in $\mathcal{F}$, 
then $(A,a)\in E$.
Denoting $SCC_\mathcal{F}(A)$ as $\hat{S}$, then Lemma \ref{the lemma for the right arrow of admissible} entails that:
\begin{itemize}
    \item $(A,a)\in D_\mathcal{F}(\hat{S},E)$, so that $(A,a)\in (D_\mathcal{F}(\hat{S},E)\cap C)$;
    \item $(A,a)$ is weakening defended by $E|_{\hat{S}}$ in $\mathcal{F}\!\downarrow_{R_{\mathcal{F}}(\hat{S},E)}$. 
\end{itemize}
Then from the hypothesis that $E|_{\hat{S}}\in \mathcal{CE}(\mathcal{F}\!\downarrow_{R_{\mathcal{F}}(\hat{S},E)},D_{\mathcal{F}}(\hat{S},E)\cap C)$, $(A,a)\in E|_{\hat{S}}$.
Consequently, $(A,a)\in E$.
\end{proof}

\subsection{Preferred Semantics}

In this subsection, we demonstrate that preferred semantics also fit the decomposition schema,
as shown by Theorem \ref{Preferred Proposition} based on the following lemma.

\begin{lem}\label{Preferred Lemma}
Let $\mathcal{F} =\langle\mathcal{A},\mathcal{R}\rangle$ be an FAF, $E\in \mathcal{AE}(\mathcal{F})$ and $S\in SCCS_{\mathcal{F}}$.
Then for any $\hat{E}\subseteq \mathcal{A}$, if $\hat{E}$ satisfies the following conditions: 
\begin{itemize}
\setlength{\itemsep}{0pt}
\setlength{\parsep}{0pt}
\setlength{\parskip}{0pt}
  \item $E|_S\subseteq \hat{E} \subseteq D_{\mathcal{F}}(S,E)$, and
  \item $\hat{E}$ is admissible in $\mathcal{F}\!\downarrow_{R_{\mathcal{F}}(S,E)}$, i.e., $\hat{E}\in \mathcal{AE}(\mathcal{F}\!\downarrow_{R_{\mathcal{F}}(S,E)})$,
\end{itemize}
then $E\cup \hat{E}$ is admissible in $\mathcal{F}$. 
\end{lem}

\begin{proof}
First, we prove that $(E\cup \hat{E})$ is conflict-free in $\mathcal{F}$.
By the hypothesis that $\hat{E}$ is admissible in $\mathcal{F}\!\downarrow_{R_{\mathcal{F}}(S,E)}$,
we derive that $\hat{E}$ is conflict-free in $\mathcal{F}\!\downarrow_{R_{\mathcal{F}}(S,E)}$.
Then $\hat{E}$ is conflict-free in $\mathcal{F}$ by Definition \ref{restriction of FAF}.
$E$ is also conflict-free in $\mathcal{F}$ by the hypothesis of $E\in \mathcal{AE}(\mathcal{F})$.
Now we have to prove that there exists no $(A,a)\in E$ and $(B,b)\in \hat{E}$ s.t. $(A,a)$ sufficiently attack $(B,b)$ or $(B,b)$ sufficiently attack $(A,a)$.
Since $E$ is admissible in $\mathcal{F}$, $(B,b)$ sufficiently attacks $(A,a)$ entails that $\exists (C,c)\in E$ sufficiently attacks $(B,b)$, and therefore we have only to consider the case that $(A,a)$ sufficiently attacks $(B,b)$.
Suppose $(A,a)$ sufficiently attacks $(B,b)$.
Then from the hypothesis that $\hat{E}\subseteq D_{\mathcal{F}}(S,E)$, $(A,a)\in E|_{S}$.
However, this is impossible since $E|_S\subseteq \hat{E}$ and $\hat{E}$ is conflict-free.
Hence, $(E\cup \hat{E})$ is conflict-free.

Now, with reference to $\mathcal{F}$, we have to prove that $\forall (A,a)\in (E\cup \hat{E})$, for any $(B,b)$ sufficiently attacks $(A,a)$, $(E\cup \hat{E})$ weakens $(B,b)$ to $(B,b')$, which tolerably attacks $(A,a)$.
In case $(A,a)\in E$, the conclusion follows from $E\in \mathcal{AE}(\mathcal{F})$.
On the other hand, if $(A,a)\in \hat{E}$, we have $B\in (outpar_{\mathcal{F}}(S)\cup S)$.
If $B\in outpar_{\mathcal{F}}(S)$, then taking into account that $\hat{E}\subseteq D_{\mathcal{F}}(S,E)$, it must be the case that $E$ weakens $(B,b)$ to $(B,b')$.  
If $B\in S$, we distinguish two cases:
\begin{enumerate}
\setlength{\itemsep}{0pt}
\setlength{\parsep}{0pt}
\setlength{\parskip}{0pt}
  \item If $(B,b)\in R_{\mathcal{F}}(S,E)$, then $(B,b)$ sufficiently attacks $(A,a)$ in $\mathcal{F}\!\downarrow_{R_{\mathcal{F}}(S,E)}$ by Definition \ref{restriction of FAF}. The hypothesis that $\hat{E}\in \mathcal{AE}(\mathcal{F}\!\downarrow_{R_{\mathcal{F}}(S,E)})$ entails that $\hat{E}$ weakens $(B,b)$ to $(B,b')$ within $\mathcal{F}\!\downarrow_{R_{\mathcal{F}}(S,E)}$, and therefore the relation also holds in $\mathcal{F}$.
  \item If $(B,b)\notin R_{\mathcal{F}}(S,E)$, then according to the definition of $R_{\mathcal{F}}(S,E)$, it must be the case that $E$ weakens $(B,b)$ to $(B,b^*)$ s.t. $(B,b^*)\in R_{\mathcal{F}}(S,E)$, which returns to case 1.
\end{enumerate}
Consequently, $(E\cup \hat{E})$ is admissible in $\mathcal{F}$, i.e., $(E\cup \hat{E})\in \mathcal{AE}(\mathcal{F})$.
\end{proof}

\begin{thm}\label{Preferred Proposition}
Given an FAF $\mathcal{F}=\langle\mathcal{A},\mathcal{R}\rangle$ and a fuzzy set $E\subseteq \mathcal{A}$,
it holds that: $\forall C\subseteq \mathcal{A}$, $E\in \mathcal{PE}(\mathcal{F}, C)$ if and only if $\forall S\in SCCS_{\mathcal{F}}$,
\begin{displaymath}
E|_S\in \mathcal{PE}(\mathcal{F}\!\downarrow_{R_{\mathcal{F}}(S,E)},D_{\mathcal{F}}(S,E)\cap C).
\end{displaymath}
\end{thm}
\begin{proof}
$\Rightarrow$: Suppose $E\in \mathcal{PE}(\mathcal{F}, C)$.
Then $E\in \mathcal{AE}(\mathcal{F}, C)$ by Definition \ref{the preferred extension in C}.
It directly follows from Theorem \ref{admissible proposition} that 
  $E|_S\in \mathcal{AE}(\mathcal{F}\!\downarrow_{R_{\mathcal{F}}(S,E)},D_{\mathcal{F}}(S,E)\cap C)$ for any $S\in SCCS_{\mathcal{F}}$.
Assuming that $\exists \hat{S}\in SCCS_{\mathcal{F}}$ such that $E|_{\hat{S}}$ is not maximal among the elements in $\mathcal{AE}(\mathcal{F}\!\downarrow_{R_{\mathcal{F}}(\hat{S},E)},D_{\mathcal{F}}(\hat{S},E)\cap C)$,
then according to Proposition \ref{the existense of preferred extensions}, there must be a fuzzy set $\hat{E}$ such that
\begin{itemize}
  \item $(E\cap \hat{S})\subsetneq \hat{E} \subseteq (D_{\mathcal{F}}(\hat{S},E)\cap C)$;
  \item $\hat{E}\in \mathcal{AE}(\mathcal{F}\!\downarrow_{R_{\mathcal{F}}(\hat{S},E)},D_{\mathcal{F}}(\hat{S},E)\cap C)$.
\end{itemize}
Lemma \ref{Preferred Lemma} entails that $(E\cup \hat{E})$ is admissible in $\mathcal{F}$.
Moreover, since both $E$ and $\hat{E}$ are contained in $C$, it follows that $(E\cup \hat{E})\subseteq C$, and therefore $(E\cup \hat{E})\in \mathcal{AE}(\mathcal{F},C)$.
However, it is evident that $E\subsetneq (E\cup \hat{E})$, which contradicts that $E$ is preferred.
Consequently, $E|_S\in \mathcal{PE}(\mathcal{F}\!\downarrow_{R_{\mathcal{F}}(S,E)},D_{\mathcal{F}}(S,E)\cap C)$.

$\Leftarrow$: Assume that $\forall S\in SCCS_{\mathcal{F}}$,
$E|_S\in \mathcal{PE}(\mathcal{F}\!\downarrow_{R_{\mathcal{F}}(\hat{S},E)},D_{\mathcal{F}}(\hat{S},E)\cap C)$.
Then we derive that $E\in \mathcal{AE}(\mathcal{F},C)$ by Theorem \ref{admissible proposition}.
To prove that $E$ is preferred, we reason by contradiction.
Suppose that $\exists E'\subseteq C$ s.t. $E\subsetneq E'$ and $E'\in \mathcal{PE}(\mathcal{F},C)$.
Since $E\subsetneq E'$, there must be at least an SCC $\hat{S}\in SCCS_{\mathcal{F}}$ s.t. $E|_{\hat{S}}\subsetneq E'|_{\hat{S}}$ and $\forall S\in sccanc_{\mathcal{F}}(\hat{S}), E'|_S=E|_S$.
Since $E'\in \mathcal{AE}(\mathcal{F},C)$, it follows from Theorem \ref{admissible proposition} that
$E'|_{\hat{S}}\in \mathcal{AE}(\mathcal{F}\!\downarrow_{R_{\mathcal{F}}(\hat{S},E')},D_{\mathcal{F}}(\hat{S},E')\cap C)$.
It is easy to see that $D_{\mathcal{F}}(\hat{S},E')=D_{\mathcal{F}}(\hat{S},E)$ and $R_{\mathcal{F}}(\hat{S},E')=R_{\mathcal{F}}(\hat{S},E)$.
Therefore, $E'|_{\hat{S}}\in \mathcal{AE}(\mathcal{F}\!\downarrow_{R_{\mathcal{F}}(\hat{S},E)},D_{\mathcal{F}}(\hat{S},E)\cap C)$, which contradicts the hypothesis that $E|_{\hat{S}}\subsetneq E'|_{\hat{S}}$ and $E|_{\hat{S}}\in \mathcal{PE}(\mathcal{F}\!\downarrow_{R_{\mathcal{F}}({\hat{S}},E)},D_{\mathcal{F}}({\hat{S}},E)\cap C)$.
Consequently, $E\in \mathcal{PE}(\mathcal{F}, C)$.
\end{proof}

\subsection{Grounded Semantics}

Finally, in this subsection, we prove that the grounded semantics also fit the decomposition schema, as shown by the following Theorem \ref{Grounded Proposition}.

\begin{thm}\label{Grounded Proposition}
Given an FAF $\mathcal{F}=\langle\mathcal{A},\mathcal{R}\rangle$ and a fuzzy set $E\subseteq \mathcal{A}$,
it holds that: $\forall C\subseteq \mathcal{A}$, $E= GE(\mathcal{F}, C)$ if and only if $\forall S\in SCCS_{\mathcal{F}}$,
\begin{displaymath}
E|_S= GE(\mathcal{F}\!\downarrow_{R_{\mathcal{F}}(S,E)},D_{\mathcal{F}}(S,E)\cap C).
\end{displaymath}
\end{thm}

\begin{proof}
$\Rightarrow$: Suppose $E = GE(\mathcal{F},C)$.
From Proposition \ref{Grounded}, $E$ is a complete fuzzy extension in $C$, i.e., $E\in \mathcal{CE}(\mathcal{F},C)$.
Therefore, Theorem \ref{Complete Proposition} entails that
\begin{center}
$\forall S\in SCCS_{\mathcal{F}}$, $E|_S\in \mathcal{CE}(\mathcal{F}\!\downarrow_{R_{\mathcal{F}}(S,E)},D_\mathcal{F}(S,E)\cap C)$.
\end{center}
Taking into account Proposition \ref{Grounded}, we have to prove that $\forall S\in SCCS_{\mathcal{F}}$, $E|_S$ is the least element in $\mathcal{CE}(\mathcal{F}\!\downarrow_{R_{\mathcal{F}}(S,E)},D_{\mathcal{F}}(S,E)\cap C)$.
Let us reason by contradiction.
Assume that $E|_S$ is not the least element for some $S\in SCCS_\mathcal{F}$.
Then we can choose an SCC $\hat{S}\in SCCS_{\mathcal{F}}$ such that
\begin{displaymath}
 \forall S\in sccanc_{\mathcal{F}}(\hat{S}), E|_S=GE(\mathcal{F}\!\downarrow _{R_{\mathcal{F}}(S,E)},D_{\mathcal{F}}(S,E)\cap C) 
\end{displaymath}
and
\begin{displaymath}
  \exists \hat{E}\subsetneq E|_{\hat{S}} \text{ s.t. } \hat{E}=GE(\mathcal{F}\!\downarrow _{R_{\mathcal{F}}(\hat{S},E)},D_{\mathcal{F}}(\hat{S},E)\cap C).
\end{displaymath}
Note that $\hat{E}$ must exist based on Proposition \ref{Grounded Lemma} and Proposition \ref{Grounded}.
Now we can construct a complete fuzzy extension $E'$ such that:
\begin{itemize}
\setlength{\itemsep}{0pt}
\setlength{\parsep}{0pt}
\setlength{\parskip}{0pt}
  \item $\forall S\in sccanc_{\mathcal{F}}(\hat{S}), E'|_S = E|_S$;
  \item $E'|_{\hat{S}} = \hat{E}$;
  \item $\forall S\in SCCS_{\mathcal{F}},E'|_S = GE(\mathcal{F}\!\downarrow_{R_{\mathcal{F}}(S,E')},D_{\mathcal{F}}(S,E')\cap C)$.
\end{itemize}
To this purpose, we first construct a fuzzy set $E^{*}$ over $(sccanc_{\mathcal{F}}(\hat{S})\cup \hat{S})$ that satisfies the first two conditions: 
$\forall S\in sccanc_{\mathcal{F}}(\hat{S})$, $E^*|_S = E|_S$ and $E^*|_{\hat{S}} = \hat{E}$.
It is obvious that  $\forall S\in(\hat{S}\cup sccanc_{\mathcal{F}}(\hat{S}))$, $D_{\mathcal{F}}(S,E^*) = D_{\mathcal{F}}(S,E)$ and $R_{\mathcal{F}}(S,E^*)=R_{\mathcal{F}}(S,E)$.
Therefore, $E^*$ also satisfies the third condition for any $S\in(\hat{S}\cup sccanc_{\mathcal{F}}(\hat{S}))$, i.e., $$\forall S\in(\hat{S}\cup sccanc_{\mathcal{F}}(\hat{S})), E^*|_S = GE(\mathcal{F}\!\downarrow_{R_{\mathcal{F}}(S,E^*)},D_{\mathcal{F}}(S,E^*)\cap C).$$
Moreover, it follows from Proposition \ref{Grounded Lemma} that 
$GE(\mathcal{F}\!\downarrow_{R_{\mathcal{F}}(S,E')},D_{\mathcal{F}}(S,E')\cap C)$ always exists for any $S\in SCCS_{\mathcal{F}}$.
Hence, $E'$ can be obtained constructively from $E^*$ by proceeding along the SCCs following $(\hat{S}\cup sccanc_{\mathcal{F}}(\hat{S}))$.
Then on the basis of Proposition \ref{Grounded}, we have that
\begin{displaymath}
    \forall S\in SCCS_{\mathcal{F}}, E'|_S\in\mathcal{CE}(\mathcal{F}\!\downarrow_{R_{\mathcal{F}}(S,E')},D_{\mathcal{F}}(S,E')\cap C).
\end{displaymath}
Thus, according to Theorem \ref{Complete Proposition}, we derive that $E'\in \mathcal{CE}(\mathcal{F},C)$.
Since  $E'|_{\hat{S}}=\hat{E}\subsetneq E|_{\hat{S}}$, it can be concluded that $E$ is not contained in $E'$, which contradicts that $E = GE(\mathcal{F},C)$.
Consequently, $\forall S\in SCCS_{\mathcal{F}}$,
$E|_S= GE(\mathcal{F}\!\downarrow_{R_{\mathcal{F}}(S,E)},D_{\mathcal{F}}(S,E)\cap C)$.

$\Leftarrow$: Suppose that $\forall S\in SCCS_{\mathcal{F}}$,
$E|_S\in GE(\mathcal{F}\!\downarrow_{R_{\mathcal{F}}(S,E)},D_{\mathcal{F}}(S,E)\cap C)$.
On the basis of Proposition \ref{Grounded}, it follows that $$\forall S\in SCCS_{\mathcal{F}},
E|_S\in \mathcal{CE}(\mathcal{F}\!\downarrow_{R_{\mathcal{F}}(S,E)},D_{\mathcal{F}}(S,E)\cap C).$$
Therefore, Theorem \ref{Complete Proposition} entails that $E\in \mathcal{CE}(\mathcal{F},C)$.
According to Proposition \ref{Grounded}, we only need to prove that $E$ is the least among the elements in $\mathcal{CE}(\mathcal{F},C)$.
Let us reason by contradiction.
Assuming that the grounded fuzzy extension $E' = GE(\mathcal{F},C)$, which must exist by Proposition \ref{Grounded Lemma}
and is a subset of $E$ by Proposition \ref{Grounded}, is strictly included in $E$.
Thus, there must be at least an SCC $S$ such that $E'|_S\subsetneq E|_S$.
Therefore, we can choose an SCC $\hat{S}$ such that 
$$\forall S\in sccanc_{\mathcal{F}}(\hat{S}), E'|_ S=E|_S \text{ and } E'|_{\hat{S}}\subsetneq E|_{\hat{S}}.$$
Moreover, since $E' = GE(\mathcal{F},C)\in \mathcal{CE}(\mathcal{F},C)$, Theorem \ref{Complete Proposition} applied to $\hat{S}$ entails that
$E'|_{\hat{S}}\in \mathcal{CE}(\mathcal{F}\!\downarrow_{R_{\mathcal{F}}(\hat{S},E')},D_{\mathcal{F}}(\hat{S},E')\cap C)$.
It is easy to see that $D_{\mathcal{F}}(\hat{S},E') = D_{\mathcal{F}}(\hat{S},E)$ and $R_{\mathcal{F}}(\hat{S},E') = R_{\mathcal{F}}(\hat{S},E)$.
Therefore, we obtain that $E'|_{\hat{S}}\in \mathcal{CE}(\mathcal{F}\!\downarrow_{R_{\mathcal{F}}(\hat{S},E)},D_{\mathcal{F}}(\hat{S},E)\cap C)$.
However, the assumption $E'|_{\hat{S}}\subsetneq E|_{\hat{S}}$ contradicts the hypothesis that $E|_{\hat{S}}\in GE(\mathcal{F}\!\downarrow_{R_{\mathcal{F}}(\hat{S},E)},D_{\mathcal{F}}(\hat{S},E)\cap C)$.
Consequently, $E= GE(\mathcal{F}, C)$.
\end{proof}

%

\subsection{SCC-recursiveness of Fuzzy Extension Semantics}

On the basis of the results obtained in the previous subsections, we characterize that admissible, complete, preferred, and grounded fuzzy extension semantics are SCC-recursive by identifying the base functions in the theorem below. 

\begin{thm}\label{Characterize semantics as SCC-recursive}
The admissible, complete, preferred and grounded semantics are SCC-recursive,
characterized by the following base functions: 
\begin{itemize}
\setlength{\itemsep}{0pt}
\setlength{\parsep}{0pt}
\setlength{\parskip}{0pt}
  \item $\mathcal{BF}_{\mathcal{AD}} (\mathcal{F},C) \equiv \mathcal{AE}(\mathcal{F},C)$;
  \item $\mathcal{BF}_{\mathcal{CO}} (\mathcal{F},C) \equiv \mathcal{CE}(\mathcal{F},C)$;
  \item $\mathcal{BF}_{\mathcal{PR}} (\mathcal{F},C) \equiv \mathcal{PE}(\mathcal{F},C)$;
  \item $\mathcal{BF}_{\mathcal{GR}} (\mathcal{F},C) \equiv \{GE(\mathcal{F},C)\}$.
\end{itemize}

\end{thm}

\begin{proof}
We prove the claim for complete semantics here; the proofs for other semantics are similar.
Now we check Definition \ref{The base definition of SCC-recursiveness in FAF} by $\mathcal{CE}(\mathcal{F},C)$.
Given an FAF $\mathcal{F}=\langle \mathcal{A}, \mathcal{R}\rangle$, as noticed before, it is easy to see that $E\in \mathcal{E}_{\mathcal{CO}}(\mathcal{F})$ if and only if $E\in \mathcal{CE}(\mathcal{F},\mathcal{A})$.
Furthermore, if $|SCCS_{\mathcal{F}}| = 1$, then $\mathcal{CE}(\mathcal{F},C)$ coincides by definition with the base function $\mathcal{BF}_{\mathcal{CO}} (\mathcal{F},C)$.
Otherwise, the decomposition schema along the SCCs follows from Theorem \ref{Complete Proposition}.
\end{proof}

\section{Illustrating Example for SCC-recursive Schema}\label{Illustrating Example}




In this section, we use examples to illustrate the process of computing semantics using the SCC-recursive schema.
Each FAF is recursively decomposed into many reduced sub-frameworks along the SCCs, enabling the efficient computation of the semantics of the original FAF based on these reduced sub-frameworks.

\begin{exmp}\label{Illustrating Example 1}
Consider an FAF $\mathcal{F}=\langle \mathcal{A},\mathcal{R}\rangle$ depicted in Figure \ref{Figure: Illustrating Example 1}, where
\begin{align*}
\mathcal{A}&=\{A_{0.8}, B_{1.0},C_{0.9}, D_{1.0},E_{1.0},F_{0.6}\}\\
\mathcal{R}&=\{A\xrightarrow{1.0} B, B\xrightarrow{1.0}C, C\xrightarrow{1.0}B, C\xrightarrow{1.0}D,D\xrightarrow{1.0}E,E\xrightarrow{1.0}F,F\xrightarrow{1.0}D\}
\end{align*}

\begin{figure}[htp]
\begin{center}
\subfloat[Before decomposition]{\label{Before decomposition}
\scalebox{0.75}{
\begin{tikzpicture}
[roundnode/.style={circle, draw=black!100, fill=white!5, thick, minimum size=2mm},
rectanglenode/.style={rectangle, draw=black!0, fill=white!5, thick, minimum size=0mm},
]
\node[roundnode]      (1) at(-4.5,0)      {$A: {0.8}$};
\node[roundnode]      (2) at(-1.5,0)      {$B: {1.0}$};
\node[roundnode]      (3) at(1.5,0)      {$C: {0.9}$};
\node[roundnode]      (4) at(4.5,0)      {$D: {1.0}$};
\node[roundnode]      (5) at(7,1.2)      {$E: {1.0}$};
\node[roundnode]      (6) at(7,-1.2)      {$F: {0.6}$};
\draw[thick, ->] (1)--(2);
\draw[thick, ->]  (-0.72,0.07)--(0.71,0.07);
\draw[thick, <-]  (-0.71,-0.07)--(0.72,-0.07);
\draw[thick, ->]  (3)--(4);
\draw[thick, ->]  (4)--(5);
\draw[thick, ->]  (5)--(6);
\draw[thick, ->]  (6)--(4);

\node at(0,-2.0) {};
\end{tikzpicture}
}
}

\subfloat[Intermediate decomposition]{
\scalebox{0.75}{
\begin{tikzpicture}
[roundnode/.style={circle, draw=black!100, fill=white!5, thick, minimum size=2mm},
rectanglenode/.style={rectangle, draw=black!0, fill=white!5, thick, minimum size=0mm},
]
\node[roundnode]      (1) at(-4.5,0)      {$A: {0.8}$};
\node[roundnode]      (2) at(-1.5,0)      {$B: {0.2}$};
\node[roundnode]      (3) at(1.5,0)      {$C: {0.9}$};
\node[roundnode]      (4) at(4.5,0)      {$D: {0.2}$};
\node[roundnode]      (5) at(7,1.2)      {$E: {1.0}$};
\node[roundnode]      (6) at(7,-1.2)      {$F: {0.6}$};
\draw[thick, ->] (4)--(5);
\draw[thick, ->] (5)--(6);
\draw[thick, ->]  (-0.72,0.07)--(0.71,0.07);
\draw[thick, <-]  (-0.71,-0.07)--(0.72,-0.07);
\node at(-4.5,-1.2) {$\mathcal{F}_1$};
\node at(0,-1.2) {$\mathcal{F}_{2}$};
\node at(5.45,-1.5) {$\mathcal{F}_{3}$};
\draw[dashed, thick] (-5.6,1.2) rectangle (-3.4,-1.5);
\draw[dashed, thick] (-2.7,1.2) rectangle (2.7,-1.5);
\draw[dashed, thick] (3.3,2.2) rectangle (8.2,-2.1);
\node at(0,-3.) {};
\end{tikzpicture}
}
}

\subfloat[After decomposition]{\label{After decomposition}
\scalebox{0.75}{
\begin{tikzpicture}
[roundnode/.style={circle, draw=black!100, fill=white!5, thick, minimum size=2mm},
rectanglenode/.style={rectangle, draw=black!0, fill=white!5, thick, minimum size=0mm},
]
\node[roundnode]      (1) at(-4.5,0)      {$A: {0.8}$};
\node[roundnode]      (2) at(-1.5,0)      {$B: {0.2}$};
\node[roundnode]      (3) at(1.5,0)      {$C: {0.9}$};
\node[roundnode]      (4) at(4.5,0)      {$D: {0.2}$};
\node[roundnode]      (5) at(7,1.2)      {$E: {0.9}$};
\node[roundnode]      (6) at(7,-1.2)      {$F: {0.2}$};
\draw[thick, ->]  (-0.72,0.07)--(0.71,0.07);
\draw[thick, <-]  (-0.71,-0.07)--(0.72,-0.07);
\node at(-4.5,-1.2) {$\mathcal{F}_1$};
\node at(0,-1.2) {$\mathcal{F}_{2}$};
\node at(4.5,-1.2) {$\mathcal{F}_{3_3}$};
\node at(7,0.) {$\mathcal{F}_{3_1}$};
\node at(7,-2.4) {$\mathcal{F}_{3_2}$};
\node at(0,-1.8) {};
\node at(0,-1.8) {};
\node at(1.8,-1.2) {};
\node at(-1.8,-1.2) {};
\end{tikzpicture}
}
}
\end{center}
\caption{SCC-recursive decomposition in Example \ref{Illustrating Example 1}}
\label{Figure: Illustrating Example 1}
\end{figure}

In this example, we discuss grounded semantics for illustration.
First, $\mathcal{F}$ can be partitioned into three SCCs:
$S_1=\{A\}, S_2=\{B,C\}, S_3=\{D,E,F\}$.
We compute the grounded fuzzy extension $E$ following the sequence of SCCs.

For the initial SCC $S_1$, it is easy to see that  
$$R_{\mathcal{F}}(S_1,E)=D_{\mathcal{F}}(S_1,E)=\{(A,0.8)\}.$$
According to Definition \ref{restriction of FAF}, it yields the first sub-framework by restricting $\mathcal{F}$
to $R_{\mathcal{F}}(S_1,E)$:
$$\mathcal{F}_1=\mathcal{F}\!\downarrow_{\mathcal{R}_\mathcal{F}(S_1,E)}=\langle \{A_{0.8}\},\varnothing\rangle.$$
Since $|SCCS_{\mathcal{F}_1}|=1$, the base function $\mathcal{BF}_{\mathcal{GR}}$ is invoked.
Consequently, we obtain $E|_{S_1}=\{(A,0.8)\}$ by applying the function $\mathcal{BF}_{\mathcal{GR}}(\mathcal{F}_1,D_{\mathcal{F}}(S_1,E))$.

Next, we consider the SCC $S_2=\{B,C\}$.
Given that $S_1$ attacks $S_2$ and $E|_{S_1}=\{(A,0.8)\}$, we derive 
    $$R_{\mathcal{F}}(S_2,E)=D_{\mathcal{F}}(S_2,E)=\{(B,0.2),(C,0.9)\}.$$
According to Definition \ref{restriction of FAF}, the restriction of $\mathcal{F}$ to $\mathcal{R}_\mathcal{F}(S_2,E)$ is
$$\mathcal{F}_2=\mathcal{F}\!\downarrow_{\mathcal{R}_\mathcal{F}(S_2,E)}=\langle \{B_{0.2},C_{0.9}\},\{B\xrightarrow{1.0}C, C\xrightarrow{1.0}B\}\rangle.$$
Since $|SCCS_{\mathcal{F}_2}|=1$, the base function $\mathcal{BF}_{\mathcal{GR}}(\mathcal{F}_2,D_{\mathcal{F}}(S_2,E))$ is invoked, which leads to $E|_{S_2}=\{(B,0.1),(C,0.8)\}$.

As far as the SCC $S_3$ is concerned, since $S_{2}$ attacks $S_3$ and $E|_{S_{2}}=\{(B,0.1),(C,0.8)\}$, we have
  \begin{align*}
      R_{\mathcal{F}}(S_3,E)&=\{(D,0.2),(E,1.0),(F,0.6)\}\\
      D_{\mathcal{F}}(S_3,E)&=\{(D,0.1),(E,1.0),(F,0.6)\}.
  \end{align*}
Given the attack from $(F,0.6)$ to $(D,0.2)$ is always tolerable, the restriction of $\mathcal{F}$ to $\mathcal{R}_\mathcal{F}(S_3,E)$ is
$$\mathcal{F}_3=\mathcal{F}\!\downarrow_{\mathcal{R}_\mathcal{F}(S_3,E)}=\langle \{D_{0.2},E_{1.0},F_{0.6}\},\{D\xrightarrow{1.0}E, E\xrightarrow{1.0}F\}\rangle.$$
Then $\mathcal{F}_3$ can be decomposed into three SCCs: $S_{3_1}=\{D\}$, $S_{3_2}=\{E\}$, $S_{3_3}=\{F\}$.
By definition, it can be concluded that
\begin{itemize}
\item for $S_{3_1}$
    \begin{itemize}
     \item  $R_{\mathcal{F}_3}(S_{3_1},E|_{S_3})=\{(D,0.2)\}$;
     \item $D_{\mathcal{F}_3}(S_{3_1},E|_{S_3})\cap D_{\mathcal{F}}(S_3,E)=\{(D,0.1)\}$;
    \item $\mathcal{F}_{3_1}=\mathcal{F}_3\!\downarrow_{\mathcal{R}_{\mathcal{F}_3}(S_{3_1},E|_{S_3})}=\langle\{D_{0.2}\},\varnothing\rangle$;
    \item $E|_{S_{3_1}}=\{(D,0.1)\}$.
    \end{itemize}
\item for $S_{3_2}$
    \begin{itemize}
     \item  $R_{\mathcal{F}_3}(S_{3_2},E|_{S_3})=\{(E,0.9)\}$;
     \item $D_{\mathcal{F}_3}(S_{3_2},E|_{S_3})\cap D_{\mathcal{F}}(S_3,E)=\{(E,0.8)\}$;
    \item $\mathcal{F}_{3_2}=\mathcal{F}_3\!\downarrow_{\mathcal{R}_{\mathcal{F}_3}(S_{3_2},E|_{S_3})}=\langle\{E_{0.9}\},\varnothing\rangle$;
    \item $E|_{S_{3_2}}=\{(E,0.8)\}$.
    \end{itemize}
\item for $S_{3_3}$
    \begin{itemize}
     \item  $R_{\mathcal{F}_3}(S_{3_3},E|_{S_3})=\{(F,0.2)\}$;
     \item $D_{\mathcal{F}_3}(S_{3_3},E|_{S_3})\cap D_{\mathcal{F}}(S_3,E)=\{(F,0.1)\}$;
    \item $\mathcal{F}_{3_1}=\mathcal{F}_3\!\downarrow_{\mathcal{R}_{\mathcal{F}_3}(S_{3_3},E|_{S_3})}=\langle\{F_{0.2}\},\varnothing\rangle$;
    \item $E|_{S_{3_3}}=\{(F,0.1)\}$.
    \end{itemize}
\end{itemize}

As a result, $\mathcal{F}$ is decomposed into five reduced sub-frameworks. Clearly, 
$$E=\{(A,0.8),(B,0.1),(C,0.8),(D,0.1),(E,0.8),(F,0.1)\}$$ is the grounded fuzzy extension of $\mathcal{F}$.
\end{exmp}

\begin{exmp}\label{Illustrating Example 2}
Consider an FAF $\mathcal{F}=\langle \mathcal{A},\mathcal{R}\rangle$ depicted in Figure \ref{Figure: Illustrating Example 2}, where
\begin{align*}
    \mathcal{A}=&  \{A_{0.8}, B_{0.8}, C_{0.6}, D_{0.9}, E_{0.8}, F_{0.8}, G_{1.0}, H_{1.0}, I_{1.0}\}\\
    \mathcal{R}=&  \{A\xrightarrow{1.0}B, B\xrightarrow{1.0}A, B\xrightarrow{1.0}C, C\xrightarrow{1.0}D, D\xrightarrow{1.0}E, E\xrightarrow{1.0}F, F\xrightarrow{1.0}C, C\xrightarrow{1.0}E, \\
    & D\xrightarrow{1.0}G, E\xrightarrow{1.0}I, G\xrightarrow{1.0}I, I\xrightarrow{1.0}G, G\xrightarrow{1.0}H, H\xrightarrow{1.0}G, I\xrightarrow{1.0}H\}
\end{align*}

\begin{figure}[htp]
\begin{center}
\subfloat[Before decomposition]{
\scalebox{0.75}{
\begin{tikzpicture}
[roundnode/.style={circle, draw=black!100, fill=white!5, thick, minimum size=2mm},
rectanglenode/.style={rectangle, draw=black!0, fill=white!5, thick, minimum size=0mm},
]
\node[roundnode]      (1) at(-4.5,0)      {$A: {0.8}$};
\node[roundnode]      (2) at(-1.5,0)      {$B: {0.8}$};
\node[roundnode]      (3) at(1.5,0)      {$C: {0.6}$};
\node[roundnode]      (4) at(4.5,0)      {$D: {0.9}$};
\node[roundnode]      (5) at(4.5,-2.5)      {$E: {0.8}$};
\node[roundnode]      (6) at(1.5,-2.5)      {$F: {0.8}$};
\node[roundnode]      (7) at(7.5,0)      {$G: {1.0}$};
\node[roundnode]      (8) at(10.5,0)      {$H: {1.0}$};
\node[roundnode]      (9) at(7.5,-2.5)      {$I: {1.0}$};

\draw[thick, ->]  (2)--(3);
\draw[thick, ->]  (3)--(4);
\draw[thick, ->]  (4)--(5);
\draw[thick, ->]  (5)--(6);
\draw[thick, ->]  (6)--(3);
\draw[thick, ->]  (3)--(5);
\draw[thick, ->]  (4)--(7);
\draw[thick, ->]  (5)--(9);
\draw[thick, ->]  (9)--(8);
\draw[thick, ->]  (-3.72,0.07)--(-2.28,0.07);
\draw[thick, <-]  (-3.72,-0.07)--(-2.28,-0.07);
\draw[thick, ->]  (8.28,0.07)--(9.7,0.07);
\draw[thick, <-]  (8.29,-0.07)--(9.7,-0.07);
\draw[thick, ->]  (7.57,-0.78)--(7.57,-1.78);
\draw[thick, <-]  (7.43,-0.78)--(7.43,-1.78);
\node at(3,-3.5) {};
\end{tikzpicture}
}
}

\subfloat[Intermediate decomposition]{
\scalebox{0.75}{
\begin{tikzpicture}
[roundnode/.style={circle, draw=black!100, fill=white!5, thick, minimum size=2mm},
rectanglenode/.style={rectangle, draw=black!0, fill=white!5, thick, minimum size=0mm},
]
\node[roundnode]      (1) at(-4.5,0)      {$A: {0.8}$};
\node[roundnode]      (2) at(-1.5,0)      {$B: {0.8}$};
\node[roundnode]      (3) at(1.5,0)      {$C: {0.2}$};
\node[roundnode]      (4) at(4.5,0)      {$D: {0.9}$};
\node[roundnode]      (5) at(4.5,-2.5)      {$E: {0.8}$};
\node[roundnode]      (6) at(1.5,-2.5)      {$F: {0.8}$};
\node[roundnode]      (7) at(7.5,0)      {$G: {0.2}$};
\node[roundnode]      (8) at(10.5,0)      {$H: {1.0}$};
\node[roundnode]      (9) at(7.5,-2.5)      {$I: {0.8}$};

\draw[thick, ->]  (3)--(4);
\draw[thick, ->]  (4)--(5);
\draw[thick, ->]  (5)--(6);
\draw[thick, ->]  (9)--(8);
\draw[thick, ->]  (-3.72,0.07)--(-2.28,0.07);
\draw[thick, <-]  (-3.72,-0.07)--(-2.28,-0.07);
\draw[thick, ->]  (8.28,0.07)--(9.7,0.07);
\draw[thick, <-]  (8.29,-0.07)--(9.7,-0.07);
\node at(-3,-1.2) {$\mathcal{F}_1$};
\node at(0.25,-1.2) {$\mathcal{F}_{2}$};
\node at(6.25,-1.2) {$\mathcal{F}_{3}$};
\draw[dashed, thick] (-5.6,1.2) rectangle (-0.4,-1.5);
\draw[dashed, thick] (-0.2,1.2) rectangle (5.6,-3.7);
\draw[dashed, thick] (5.8,1.2) rectangle (11.6,-3.7);

\node at(3,-4.) {};
\end{tikzpicture}
}
}

\subfloat[After decomposition]{
\scalebox{0.75}{
\begin{tikzpicture}
[roundnode/.style={circle, draw=black!100, fill=white!5, thick, minimum size=2mm},
rectanglenode/.style={rectangle, draw=black!0, fill=white!5, thick, minimum size=0mm},
]
\node[roundnode]      (1) at(-4.5,0)      {$A: {0.8}$};
\node[roundnode]      (2) at(-1.5,0)      {$B: {0.8}$};
\node[roundnode]      (3) at(1.5,0)      {$C: {0.2}$};
\node[roundnode]      (4) at(4.5,0)      {$D: {0.8}$};
\node[roundnode]      (5) at(4.5,-2.5)      {$E: {0.2}$};
\node[roundnode]      (6) at(1.5,-2.5)      {$F: {0.8}$};
\node[roundnode]      (7) at(7.5,0)      {$G: {0.2}$};
\node[roundnode]      (8) at(10.5,0)      {$H: {0.2}$};
\node[roundnode]      (9) at(7.5,-2.5)      {$I: {0.8}$};
\draw[thick, ->]  (-3.72,0.07)--(-2.28,0.07);
\draw[thick, <-]  (-3.72,-0.07)--(-2.28,-0.07);

\node at(-3,-1.2) {$\mathcal{F}_1$};
\node at(1.5,-1.2) {$\mathcal{F}_{2_1}$};
\node at(4.5,-1.2) {$\mathcal{F}_{2_2}$};
\node at(4.5,-3.7) {$\mathcal{F}_{2_3}$};
\node at(1.5,-3.7) {$\mathcal{F}_{2_4}$};
\node at(7.5,-1.2) {$\mathcal{F}_{3_{2_2}}$};
\node at(7.5,-3.7) {$\mathcal{F}_{3_1}$};
\node at(10.5,-1.2) {$\mathcal{F}_{3_{2_1}}$};
\node at(3,-3.5) {};
\end{tikzpicture}
}
}
\end{center}
\caption{SCC-recursive decomposition in Example \ref{Illustrating Example 2}}
\label{Figure: Illustrating Example 2}
\end{figure}

In this example, we compute a preferred fuzzy extension $E$ of $\mathcal{F}$.
First, $\mathcal{F}$ can be partitioned into three SCCs: $S_1=\{A,B\}$, $S_2=\{C,D,E,F\}$, $S_3=\{G,H,I\}$.
Subsequently, we compute $E$ following the sequence of these SCCs.

For the initial SCC $S_1=\{A,B\}$, it is easy to see that 
$$D_{\mathcal{F}}(S_1,E)=R_{\mathcal{F}}(S_1,E)=\{(A,0.8),(B,0.8)\},$$ 
yielding the first sub-framework by restricting $\mathcal{F}$ to $R_{\mathcal{F}}(S_1,E)$:
$$\mathcal{F}_1=\mathcal{F}\!\downarrow_{R_{\mathcal{F}}(S_1,E)}=\langle \{A_{0.8},B_{0.8}\},\{A\xrightarrow{1.0}B, B\xrightarrow{1.0}A\}\rangle.$$
Since $|SCCS_{\mathcal{F}_1}|=1$, the base function $\mathcal{BF}_{\mathcal{PR}}$ is invoked.
There are many results for selection, which potentially lead to different decomposition.
We choose $E|_{S_1}=\{(A,0.2),(B,0.8)\}$ for illustration.

Next, we consider the SCC $S_2=\{C,D,E,F\}$. 
Given that $S_1$ attacks $S_2$ and $E|_{S_1}=\{(A,0.2),(B,0.8)\}$,
we have 
    $$R_{\mathcal{F}}(S_2,E)=D_{\mathcal{F}}(S_2,E)=\{(C,0.2),(D,0.9),(E,0.8),(F,0.8)\}.$$
Evidently, the attacks from $(C,0.2)$ to $(E,0.8)$ and $(F,0.8)$ to $(C,0.2)$ are always tolerable, and therefore according to Definition \ref{restriction of FAF}, the restriction of $\mathcal{F}$ to  $R_{\mathcal{F}}(S_2,E)$ is 
$$\mathcal{F}_2=\mathcal{F}\!\downarrow_{{R_{\mathcal{F}}}(S_2,E)}=\langle\{C_{0.2}, D_{0.9},E_{0.8},F_{0.8}\},\{C\xrightarrow{1.0}D, D\xrightarrow{1.0}E, E\xrightarrow{1.0}F\}\rangle.$$
Then $\mathcal{F}_2$ can be recursively decomposed into four SCCs: $S_{2_1}=\{C\}$, $S_{2_2}=\{D\}$, $S_{2_3}=\{E\}$, $S_{2_4}=\{F\}$.
It can be concluded that 
\begin{itemize}
\item for $S_{2_1}$ 
\begin{itemize}
    \item $R_{\mathcal{F}_2}(S_{2_1},E|_{S_2})=\{(C,0.2)\}$;
    \item $D_{\mathcal{F}_2}(S_{2_1},E|_{S_2})\cap D_{\mathcal{F}}(S_2,E)=\{(C,0.2)\}$;
    \item $\mathcal{F}_{2_1}=\mathcal{F}_2\!\downarrow_{\mathcal{R}_{\mathcal{F}_2}(S_{2_1},E|_{S_2})}=\langle\{C_{0.2}\},\varnothing\rangle$;
    \item $E|_{S_{2_1}}=\{(C,0.2)\}$.
    \end{itemize}
\item for $S_{2_2}$
\begin{itemize}
    \item $R_{\mathcal{F}_2}(S_{2_2},E|_{S_2})=\{(D,0.8)\}$;
    \item $D_{\mathcal{F}_2}(S_{2_2},E|_{S_2})\cap D_{\mathcal{F}}(S_2,E)=\{(D,0.8)\}$;
    \item  $\mathcal{F}_{2_2}=\mathcal{F}_2\!\downarrow_{\mathcal{R}_{\mathcal{F}_2}(S_{2_2},E|_{S_2})}=\langle\{D_{0.8}\},\varnothing\rangle$;
    \item $E|_{S_{2_2}}=\{(D,0.8)\}$.
    \end{itemize}
\item for $S_{2_3}$
\begin{itemize}
    \item $R_{\mathcal{F}_2}(S_{2_3},E|_{S_2})=\{(E,0.2)\}$;
    \item $D_{\mathcal{F}_2}(S_{2_3},E|_{S_2})\cap D_{\mathcal{F}}(S_2,E)=\{(E,0.2)\}$;
    \item  $\mathcal{F}_{2_3}=\mathcal{F}_2\!\downarrow_{\mathcal{R}_{\mathcal{F}_2}(S_{2_3},E|_{S_2})}=\langle\{E_{0.2}\},\varnothing\rangle$;
    \item $E|_{S_{2_3}}=\{(E,0.2)\}$.
    \end{itemize}
\item for $S_{2_4}$
\begin{itemize}
    \item $R_{\mathcal{F}_2}(S_{2_4},E|_{S_2})=\{(F,0.8)\}$;
    \item $D_{\mathcal{F}_2}(S_{2_4},E|_{S_2})\cap D_{\mathcal{F}}(S_2,E)=\{(F,0.8)\}$;
    \item  $\mathcal{F}_{2_4}=\mathcal{F}_2\!\downarrow_{\mathcal{R}_{\mathcal{F}_2}(S_{2_4},E|_{S_2})}=\langle\{F_{0.8}\},\varnothing\rangle$;
    \item $E|_{S_{2_4}}=\{(F,0.8)\}$.
    \end{itemize}
\end{itemize}
Consequently, $E|_{S_2}=\{(C,0.2), (D,0.8),$ $(E,0.2), (F,0.8)\}$.

As far as the SCC $S_3$ is concerned, from that $S_{2}$ attack $S_3$ and $E|_{S_2}=\{(C,0.2), (D,0.8),$ $(E,0.2), (F,0.8)\}$, we derive that 
   $$R_{\mathcal{F}}(S_3,E)=D_{\mathcal{F}}(S_3,E)=\{(G,0.2),(H,1.0),(I,0.8)\}.$$
Since the attacks between $(G,0.2)$ and $(I,0.8)$ are always tolerable, according to Definition \ref{restriction of FAF}, the restriction of $\mathcal{F}$ to  $R_{\mathcal{F}}(S_3,E)$ is 
$$\mathcal{F}_3=\mathcal{F}\!\downarrow_{{R_{\mathcal{F}}(S_3,E)}}=\langle\{G_{0.2},H_{1.0},I_{0.8}\},\{G\xrightarrow{1.0}H,H\xrightarrow{1.0}G,I\xrightarrow{1.0}H\}\rangle.$$
Similarly, $\mathcal{F}_3$ can be recursively decomposed into two SCCs: $S_{3_1}=\{I\}$, $S_{3_2}=\{G,H\}$.
\begin{itemize}
\item For $S_{3_1}$
\begin{itemize}
    \item $R_{\mathcal{F}_3}(S_{3_1},E|_{S_3})=\{(I,0.8)\}$;
    \item $D_{\mathcal{F}_3}(S_{3_1},E|_{S_3})\cap D_{\mathcal{F}}(S_3,E)=\{(I,0.8)\}$;
    \item  $\mathcal{F}_{3_1}=\mathcal{F}_3\!\downarrow_{\mathcal{R}_{\mathcal{F}_3}(S_{3_1},E|_{S_3})}=\langle\{I_{0.8}\},\varnothing\rangle$;
    \item $E|_{S_{3_1}}=\{(I,0.8)\}$.
    \end{itemize}
\end{itemize}
For $S_{3_2}$, since $S_{3_1}$ attacks $S_{3_2}$ and $E|_{S_{3_1}}=\{(I,0.8)\}$, we obtain
    $$R_{\mathcal{F}_3}(S_{3_2},E|_{S_3})=D_{\mathcal{F}_3}(S_{3_2},E|_{S_3})=\{(G,0.2),(H,0.2)\}.$$
Clearly the attacks between $(G,0.2)$ and $(H,0.2)$ are always tolerable, therefore, the restriction of $\mathcal{F}_3$ to $R_{\mathcal{F}}(S_{3_2},E|_{S_3})$ is $$\mathcal{F}_{3_2}=\mathcal{F}_3\!\downarrow_{{R_{\mathcal{F}_3}}(S_{3_2},E|_{S_3})}=\langle\{G_{0.2},H_{0.2}\},\varnothing\rangle.$$
Then $\mathcal{F}_{3_2}$ can be recursively partitioned into $S_{3_{2_1}}=\{H\}$ and $S_{3_{2_2}}=\{G\}$.
Similar to the above analysis, we derive that 
\begin{itemize}
\item for $S_{3_{2_1}}$
\begin{itemize}
    \item $R_{\mathcal{F}_{3_2}}(S_{3_{2_1}},E|_{S_{3_2}})=\{(H,0.2)\}$;
    \item $D_{\mathcal{F}_{3_2}}(S_{3_{2_1}},E|_{S_{3_2}})\cap D_{\mathcal{F}_3}(S_{3_1},E|_{S_3})\cap D_{\mathcal{F}}(S_3,E)=\{(H,0.2)\}$;
    \item $\mathcal{F}_{3_{2_1}}=\mathcal{F}_{3_2}\downarrow_{\mathcal{R}_{\mathcal{F}_{3_2}}(S_{3_{2_1}},E|_{S_{3_2}})}=\langle\{H_{0.2}\},\varnothing\rangle$;
    \item $E|_{S_{3_{2_1}}}=\{(H,0.2)\}$.
    \end{itemize}
\item for $S_{3_{2_2}}$
    \begin{itemize}
    \item $R_{\mathcal{F}_{3_2}}(S_{3_{2_2}},E|_{S_{3_2}})=\{(G,0.2)\}$;
    \item $D_{\mathcal{F}_{3_2}}(S_{3_{2_2}},E|_{S_{3_2}})\cap D_{\mathcal{F}_3}(S_{3_1},E|_{S_3})\cap D_{\mathcal{F}}(S_3,E)=\{(G,0.2)\}$;
    \item $\mathcal{F}_{3_{2_2}}=\mathcal{F}_{3_2}\downarrow_{\mathcal{R}_{\mathcal{F}_{3_2}}(S_{3_{2_2}},E|_{S_{3_2}})}=\langle\{G_{0.2}\},\varnothing\rangle$;
    \item $E|_{S_{3_{2_2}}}=\{(G,0.2)\}$.
    \end{itemize}
\end{itemize}
Consequently, $E|_{S_3}=\{(G,0.2),(H,0.2),(I,0.8)\}$.

As a result, the combination fuzzy extension 
$$E=\{(A,0.2),(B,0.8),(C,0.2),(D,0.8),(E,0.2),(F,0.8),(G,0.2),(H,0.2),(I,0.8)\}$$  
is a preferred fuzzy extension of $\mathcal{F}$.
\end{exmp}

\section{Discussion and Conclusion}\label{Discussion and Conclusion}

SCC-recursiveness was proposed in \cite{baroni2005scc} as a powerful schema for characterizing semantics through the decomposition of AF along SCCs.
This schema has been extensively studied in the literature.

Firstly, it has proven useful in developing algorithms for solving semantics.
Cerutti et al. designed an SCC-recursive algorithm for computing preferred semantics in \cite{cerutti2014scc} and further exploited the parallel computation in \cite{cerutti2015exploiting}.
In \cite{Baroni2014On}, Baroni et al. proposed an incremental computation algorithm for solving semantics in the dynamics of AF based on the schema.

Secondly, this schema has facilitated the exploration of new semantics.
In \cite{baroni2005scc}, Baroni et al. proposed \emph{CF2} and \emph{AD2} semantics by incorporating the concept of \emph{conflict-freeness} and \emph{admissibility} with the SCC-recursive schema.
In \cite{dvovrak2016stage}, Dvo{\v{r}}{\'a}k and Gaggl proposed \emph{stage2} semantics by combining stage semantics with the schema.

Thirdly, the feasibility of SCC-recursiveness to various semantics has attracted considerable attention.
In \cite{villata2011attack}, Villata et al. proposed the so-called \emph{attack semantics} and defined an SCC-recursive schema for this semantics using \emph{attack labelings}.
In \cite{Dauphin2020principle}, Dauphin et al. demonstrated that the SCC-recursive schema is inapplicable to \emph{weakly admissible}, \emph{weakly complete} and \emph{weakly grounded semantics}, 
whereas Dvo{\v{r}}{\'a}k et al. confirmed its applicability to \emph{weakly preferred semantics} in \cite{dvovrak2022recursion}.

Finally, SCC-recursiveness has also been extended to various qualitative frameworks.
In \cite{baumann2017study}, Baumann and Spanring investigated it in \emph{unrestricted AF}.
In \cite{yu2021principle}, Yu et al. examined it in \emph{abstract agent AF}. 
In \cite{gaggl2021decomposition}, Gaggl et al. studied it in \emph{abstract dialectical frameworks}.
In \cite{dvovrak2024principles}, Dvo{\v{r}}{\'a}k et al. explored it in \emph{AF with collective attacks}.

Despite substantial contributions in the literature, the exploration of SCC-recursiveness in quantitative settings has been relatively neglected.
In this paper, we demonstrated that SCC-recursiveness can be applied to characterize fuzzy extension semantics in fuzzy AF.
To achieve this, we tailored the existing SCC-recursive schema, enabling the characterization of fuzzy extension semantics---including \emph{admissible}, \emph{complete}, \emph{grounded} and \emph{preferred}---through the recursive decomposition of an FAF along its SCCs.
Our contributions are twofold.
Theoretically, we showed that SCC-recursiveness provides an alternative approach to characterize fuzzy extension semantics, offering a deep understanding and better insight into these semantics.
Practically, we provided a sound and complete algorithm for computing fuzzy extension semantics.
As illustrated by examples, this algorithm naturally reduces computational efforts when dealing with a large number of SCCs.

Future research can be approached in several ways. 
First, it is worth investigating the development of specific algorithms for computing fuzzy extension semantics based on the SCC-recursive schema. 
Second, utilizing the schema to explore new semantics in fuzzy AF offers significant potential.
Third, investigating SCC-recursiveness in other quantitative settings, such as probabilistic AF \cite{hunter2021probabilistic}, is a desirable endeavor.

\bibliographystyle{plain}
\bibliography{Reference}

\end{document}